\documentclass{article}
\usepackage[preprint]{neurips_2026}

\usepackage[utf8]{inputenc} 
\usepackage[T1]{fontenc}    
\usepackage{hyperref}       
\usepackage{url}            
\usepackage{booktabs}       
\usepackage{amsfonts}       
\usepackage{nicefrac}       
\usepackage{microtype}      
\usepackage{xcolor}         

\usepackage{xspace}
\usepackage{graphicx}
\usepackage{tcolorbox}
\usepackage{algorithm}
\usepackage{algpseudocode}
\tcbuselibrary{breakable}
\usepackage{colortbl}
\usepackage{xcolor}
\usepackage{wrapfig}

\usepackage{tikz}
\usepackage{subcaption}

\usepackage{amsmath}
\usepackage{amsthm}
\newtheorem{proposition}{Proposition}
\newtheorem{lemma}{Lemma}
\newtheorem{corollary}{Corollary}
\usepackage{multirow}
\usepackage{multicol}
\usepackage[inline]{enumitem} 
\usepackage{booktabs}
\usepackage{verbatim}

\usepackage[capitalise,noabbrev]{cleveref}
\crefname{section}{\S}{\S\S}
\Crefname{section}{\S}{\S\S}
\crefname{figure}{Fig.}{Figs.}
\Crefname{figure}{Fig.}{Figs.}
\crefname{table}{Tab.}{Tabs.}
\Crefname{table}{Tab.}{Tabs.}
\crefname{appendix}{Appx.}{Appxs.}
\Crefname{appendix}{Appx.}{Appxs.}

\usepackage{appendix}

\usepackage{threeparttable}   
\usepackage{makecell}         
\usepackage{pifont}           

\newlist{inparaenum}{enumerate*}{1}
\setlist[inparaenum,1]{label=\arabic*), itemjoin={{ }}, itemjoin*={{ }}}

\def\eg{{\em e.g.,}\xspace}
\def\ie{{\em i.e.,}\xspace}
\def\versus{{\em v.s.}\xspace}

\definecolor{lightgray}{HTML}{F0F0EB}
\definecolor{lightorange}{HTML}{FFD2A4}
\definecolor{lightgreen}{HTML}{A4FFAE}
\definecolor{stronggreen}{HTML}{3EFF54}
\definecolor{lightred}{HTML}{FFA4A4}
\definecolor{strongred}{HTML}{FF7171}

\definecolor{bhred}{HTML}{BE1E2D}
\definecolor{bhblue}{HTML}{0563BA}
\definecolor{bhorange}{HTML}{E8972E}

\definecolor{lgra}{HTML}{F0F0EB}
\definecolor{lora}{HTML}{FFD2A4}
\definecolor{lgre}{HTML}{A4FFAE}
\definecolor{sgre}{HTML}{3EFF54}
\definecolor{lred}{HTML}{FFA4A4}
\definecolor{sred}{HTML}{FF7171}
\definecolor{lblu}{HTML}{A4C7FF}

\definecolor{background-prompt}{HTML}{EFEFEA}
\definecolor{background-disclaimer}{HTML}{EBDBBC}
\definecolor{border}{HTML}{262625}
\definecolor{border-light}{HTML}{D9D9CD}
\definecolor{background-takeaway}{HTML}{EBDBBC}

\usepackage{newunicodechar}
\newunicodechar{♫}{\ensuremath{\flat}}

\def\methodnamecap{Logit-contribution scoring\xspace}
\def\methodnametitle{Logit-Contribution Scoring\xspace}
\def\methodnameabbrev{\textsc{LOCOS}\xspace}

\newcommand{\repo}[1]{\url{https://github.com/aryopg/locos}}
\newcommand{\datarepo}[1]{\url{}}
\newcommand{\examplepage}[1]{\url{}}

\newcommand{\prompt}[2]{
\begin{tcolorbox}[colback=background-prompt, colframe=white, 
  left=2pt,
  right=2pt,
  top=2pt,
  bottom=2pt]
{\textbf{\emph{#1:}}\\{#2}}
\end{tcolorbox}
}

\newcounter{takeaway}
\newcommand{\newtakeaway}[1]{\refstepcounter{takeaway}
\begin{tcolorbox}[colback=background-takeaway, colframe=background-takeaway, 
  left=2pt,
  right=2pt,
  top=2pt,
  bottom=2pt]
{\textbf{\emph{Takeaway \thetakeaway:} }{#1}}
\end{tcolorbox}
}

\providecommand{\answerYes}[1][]{\textbf{Yes}}
\providecommand{\answerNo}[1][]{\textbf{No}}
\providecommand{\answerNA}[1][]{\textbf{N/A}}

\usepackage{listings}
\usepackage{xcolor}

\lstdefinelanguage{Jinja}{
    morestring=[b]",
    morestring=[b]',
    morecomment=[l]{\#},
    moredelim=[s][\color{red!70!black}]{\{\{}{\}\}},
    moredelim=[s][\color{blue!70!black}]{\{\%}{\%\}},
    moredelim=[s][\color{green!70!black}]{\{\#}{\#\}},
    keywords={endif,endfor,elif},
    keywordstyle=\color{purple}\bfseries,
    sensitive=true
}

\lstnewenvironment{jinjacode}[1][]{
    \lstset{
        language=Jinja,
        basicstyle=\small\ttfamily,
        backgroundcolor=\color{gray!5},
        frame=single,
        frameround=tttt,
        breaklines=true,
        breakatwhitespace=true,
        showstringspaces=false,
        tabsize=2,
        captionpos=b,
        #1
    }
}{}

\algdef{SE}[FUNCTION]{Function}{EndFunction}%
[2]{\algorithmicfunction\ \textproc{#1}\ifthenelse{\equal{#2}{}}{}{(#2)}}%
{\algorithmicend\ \algorithmicfunction}

\newcommand{\revision}[1]{#1}

\title{\methodnametitle Identifies\\Non-Literal Retrieval Heads}

\author{
    Aryo Pradipta Gema$^{\textcolor{bhred}{Q}}$\qquad
    Beatrice Alex$^{\textcolor{bhorange}{K}}$\qquad
    Pasquale Minervini$^{\textcolor{bhred}{Q},\textcolor{bhblue}{V}}$\\
    $^{\textcolor{bhred}{Q}}$University of Edinburgh\qquad
    $^{\textcolor{bhorange}{K}}$Heriot-Watt University\qquad
    $^{\textcolor{bhblue}{V}}$Miniml.AI\\
    \texttt{\{aryo.gema, p.minervini\}@ed.ac.uk}\quad\quad 
    \texttt{b.alex@hw.ac.uk}\\
}

\begin{document}

\maketitle

\begin{abstract}
\looseness-1
In long-context use, large language models frequently synthesize answers from the meaning of a relevant context span rather than literally copy-pasting them.
Identifying which attention heads perform this synthesis matters for interpreting long-context model behavior.
Yet existing detectors miss these heads by construction: they reward heads whose attended token matches the generated token, a literal-copy criterion that captures \emph{where} a head reads but not \emph{what} it writes through its output-value (OV) circuit, the very mechanism that carries non-literal retrieval.
We introduce \methodnametitle (\methodnameabbrev), a write-aware detector that scores each head by the projection of its OV-circuit output onto the answer-token unembedding direction, contrasting needle and off-needle source positions in a single forward pass.
Across three model families (Qwen3, Gemma-3, OLMo-3.1), mean-ablating the top \methodnameabbrev heads on the NoLiMa non-literal retrieval benchmark collapses ROUGE-L at lower head counts than prior attention-based detections; on Qwen3-8B, ablating 50 heads drives ROUGE-L from $0.401$ to $0.000$ while the strongest baseline still retains $0.292$.
The selected heads are retrieval-specific: parametric recall and arithmetic reasoning stay at baseline under the same ablation.
On Qwen3-8B, the same ablation also drops MuSiQue from $0.55$ to $0.08$ and BABILong from $0.62$ to $0.20$, while a random-heads control stays within $0.05$ of baseline.
\begin{center}
    \href{https://github.com/aryopg/locos}{
        \raisebox{-0.2\height}{\includegraphics[height=1em]{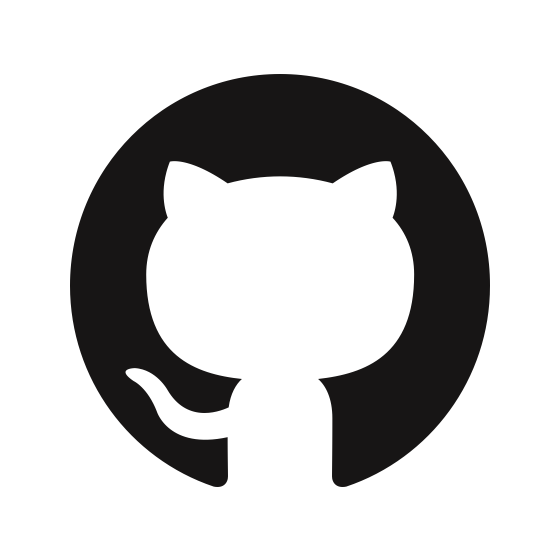}}
        \hspace{0.2em}\texttt{locos}
    }
    \hspace{1em}
    \href{https://huggingface.co/datasets/aryopg/locos-results}{
        \raisebox{-0.2\height}{\includegraphics[height=1em]{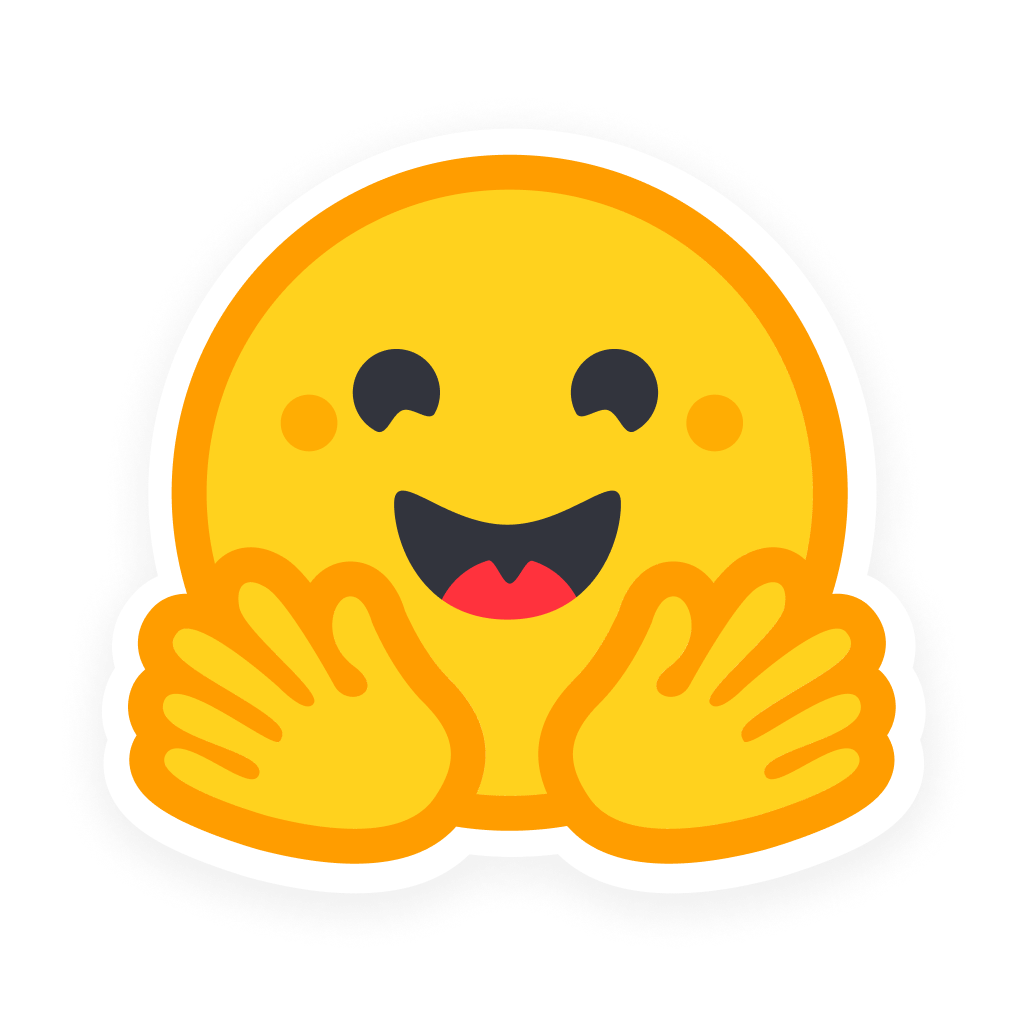}}
    \hspace{0.2em}\texttt{locos-results}
    }
\end{center}
\end{abstract}

\section{Introduction}
\label{sec:introduction}

\looseness-1
The ability of large language models (LLMs) to retrieve information from their input context, rather than relying on memorized parametric knowledge, depends on a sparse set of attention heads known as \emph{retrieval heads}~\citep{wu2025retrieval}, which build on earlier mechanistic work on induction heads~\citep{elhage2021mathematical,olsson2022inductionheads}.
However, in practice, context retrieval is rarely literal copy-paste: a user's question may share no lexical overlap with the relevant passage, and the model must identify the relevant snippet, parse its meaning, and synthesize an answer from it (as illustrated in \cref{fig:literal_vs_nonliteral}).
Yet all existing identification methods, whether via token-matching heuristics~\citep{wu2025retrieval} or weighted attention accumulation~\citep{fu2025headkv,lin2025compresskv}, evaluate heads on literal copying tasks and share a common observable: each head's attention pattern over source positions.
This observable captures \emph{where} a head allocates attention, but does not capture \emph{what information} the head propagates through its output-value (OV) circuit, the very mechanism by which non-literal retrieval is accomplished.
%
%
Two heads with identical attention patterns but different OV circuits can propagate entirely different information from the same positions.
For literal retrieval, where the attended token is the answer token, attention and OV output are trivially aligned, and attention-based scoring works.
In non-literal retrieval, a head may attend to ``Eiffel Tower'' while writing the ``Yuki'' direction to the residual stream.
Attention-based scoring sees the read site (``Eiffel Tower''); the OV circuit determines the write content (``Yuki''), and the two need not agree.
This distinction matters in practice, and no prior retrieval-head detector makes this distinction.

\begin{wrapfigure}{r}{0.45\textwidth}
  \centering
  \vspace{-0.05em}
  \includegraphics[width=0.45\textwidth]{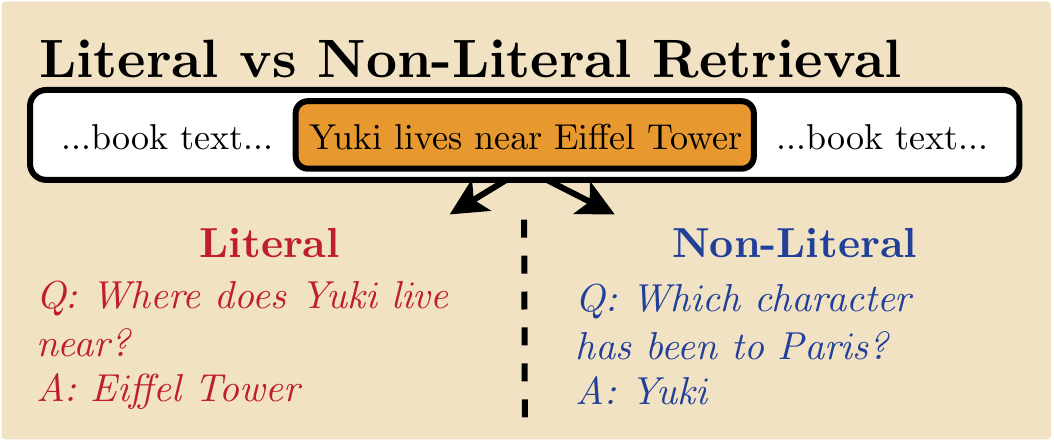}
  \caption{
  \looseness-1
  \textbf{Non-literal retrieval requires synthesis.}
  The same context answers two questions differently: a literal question requires reading ``Eiffel Tower'' directly from the needle, while a non-literal question must produce ``Yuki'' after synthesizing the context.
  }
  \label{fig:literal_vs_nonliteral}
  \vspace{-1.5em}
\end{wrapfigure}
Our method, \methodnametitle (\methodnameabbrev), measures how each attention head contributes to the correct answer token in the unembedding space (See \cref{fig:overview}).
%
For each head at each source position, the method computes the scalar projection of the head's weighted OV circuit output onto the correct answer's unembedding vector.
Aggregation uses \emph{spatial contrast}: logit contributions from needle positions are compared against length-normalized off-needle contributions within a single decoding step.
%
%
The method requires only a single forward pass per probing trial.

\looseness-1
Ablation experiments on six configurations spanning three model families (\ie Qwen3 (8B, 14B, 32B)~\citep{qwen3technicalreport}, Gemma-3 (12B, 27B)~\citep{gemma_2025}, and OLMo-3.1 (32B)~\citep{olmo2025olmo3}) on the NoLiMa non-literal retrieval benchmark heldout set~\citep{modarressi2025nolima} validate the method: mean-ablating the top-ranked \methodnameabbrev heads produces a steeper ROUGE-L degradation curve than all evaluated baselines (\cref{sec:ablation-comparison}).
On Qwen3-8B, ablating the top-50 \methodnameabbrev heads collapses ROUGE-L from $0.401$ to $0.000$, while the strongest attention-based baseline retains $0.292$ at the same depth.
Control experiments confirm retrieval specificity: parametric recall and arithmetic reasoning remain intact under the same ablation (\cref{sec:specificity}).
The same ablation also degrades downstream long-context performance, most strongly on the Qwen3 family (\cref{sec:downstream}): on Qwen3-8B, ablating the top-$50$ \methodnameabbrev heads drops MuSiQue accuracy from $0.55$ to $0.08$ and BABILong from $0.62$ to $0.20$\revision{; transfer is most consistent on the Qwen3 family, and the ranking against attention-based baselines is benchmark-dependent on Gemma-3 and OLMo-3.1 (\cref{sec:downstream})}.

\section{Background}
\label{sec:background}

\textbf{Notation.}
Consider a transformer~\citep{vaswani2017attention} with $L$ layers and $H$ attention heads per layer, head dimension $d_h$, and model dimension $d = H \cdot d_h$.
Head $(l,h)$ attends to source positions with weights $\alpha^{(l,h)}_{t,j}$, reads value vectors $\mathbf{v}^{(l,h)}_{t,j} \in \mathbb{R}^{d_h}$, and writes to the residual stream via its output projection $W_O^{(l,h)} \in \mathbb{R}^{d \times d_h}$.
The per-position output of head $(l,h)$ from source position $j$ is:
\begin{equation}\label{eq:per-position-output}
    \mathbf{o}^{(l,h)}_{t,j} \;=\; \alpha^{(l,h)}_{t,j} \cdot W_O^{(l,h)}\, \mathbf{v}^{(l,h)}_{t,j} \;\in\; \mathbb{R}^d\,.
\end{equation}
The unembedding matrix $W_U \in \mathbb{R}^{|\mathcal{V}| \times d}$ maps the residual stream to logits; $\mathbf{u}_y \in \mathbb{R}^d$ is its $y$-th row.

\textbf{Attention heads as read-and-write circuits.}
An attention head decomposes into a \emph{QK circuit} that determines $\alpha^{(l,h)}_{t,j}$ (where the head reads) and an \emph{OV circuit} that maps each source value through $W_O$ to the residual stream (what the head writes)~\citep{elhage2021mathematical}.
The full output of head $(l,h)$ at step $t$ is $\sum_j \mathbf{o}^{(l,h)}_{t,j}$, and the next-token logits come from mapping the final-layer residual stream through $W_U$.
A head is useful for retrieval only when both stages agree: the QK circuit must select the right source positions \emph{and} the OV circuit must write an answer-aligned update.

\looseness-1
\textbf{Induction heads and retrieval heads.}
Induction heads implement the literal copy pattern $[A][B]\ldots[A]\mapsto[B]$~\citep{olsson2022inductionheads}: they place high $\alpha^{(l,h)}_{t,j}$ on a previously matching position and write an output whose projection onto $\mathbf{u}_{y_t}$ is large because the attended token \emph{is} the next token.
Retrieval heads, identified on literal needle-in-a-haystack (NIAH) prompts by \citet{wu2025retrieval}, generalize this picture to long-context factuality. Their detection procedure rewards heads whose argmax attention falls inside the needle and the attended token matches the generated token, a literal-copy criterion.

\looseness-1
\textbf{Literal versus non-literal retrieval.}
In a NIAH setup~\citep{kamradt2023niah}, a \emph{needle} is a short answer-bearing span inserted into a longer distractor context (the \emph{haystack}); we index it as $[s_\tau, e_\tau)$ for trial $\tau$ and refer to all other source positions as \emph{off-needle}.
In literal NIAH, the answer token appears in the needle, so large needle attention $\alpha^{(l,h)}_{t,j}$ usually coincides with a large logit-relevant write $\mathbf{u}_{y_t}^\top \mathbf{o}^{(l,h)}_{t,j}$. NoLiMa breaks this equivalence~\citep{modarressi2025nolima}: the answer must be recovered from the meaning of the needle and may share no lexical overlap with it. A head can then retrieve non-literal information by attending to a semantically relevant phrase inside $[s_\tau, e_\tau)$ and writing an answer-aligned direction even when no attended token matches $y_t$. This is the regime targeted by \methodnameabbrev.

\section{\methodnametitle}
\label{sec:method}

\begin{figure}[t]
  \centering
  \includegraphics[width=\linewidth]{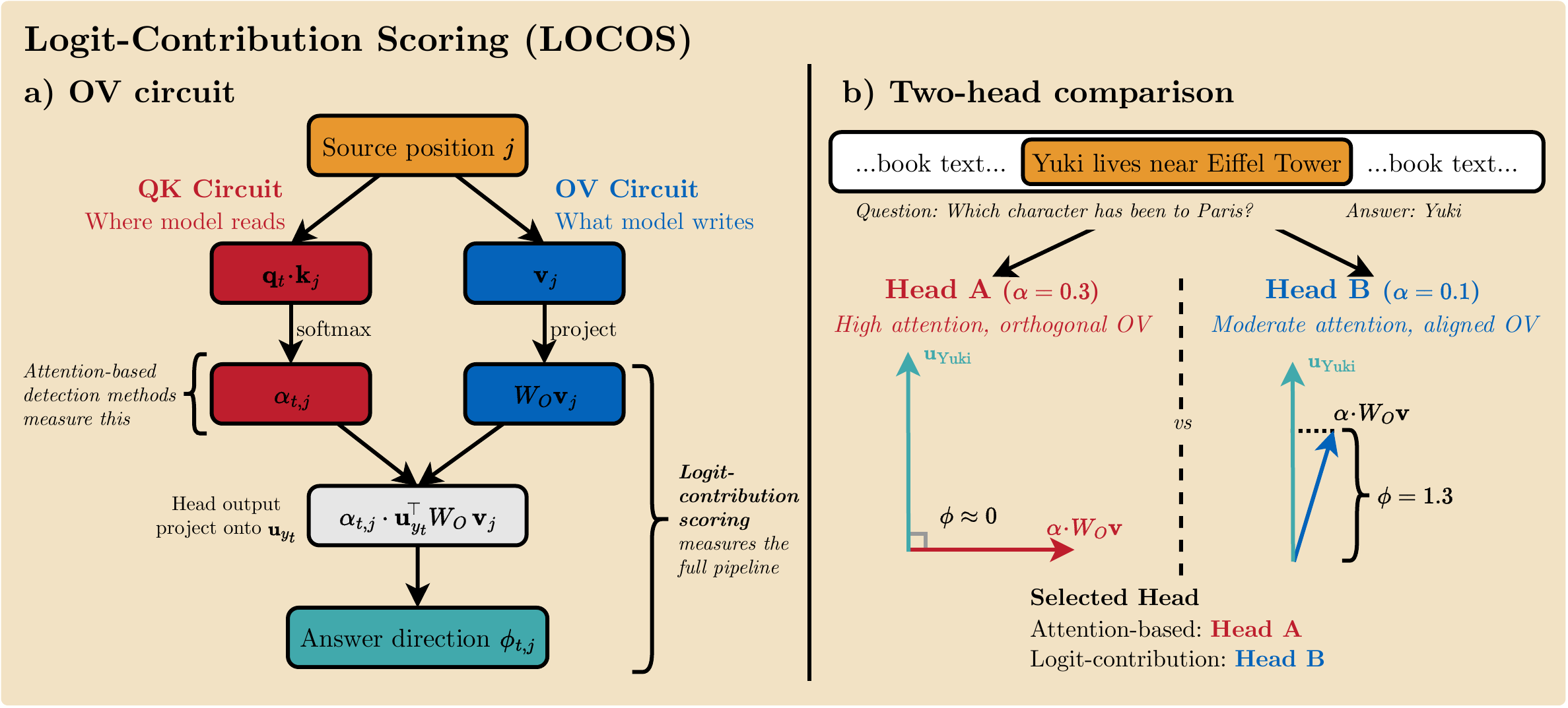}
  \caption{\textbf{An attention head has two circuits: where it reads (QK) and what it writes (OV). \methodnametitle uses the OV circuit to identify \emph{non-literal retrieval heads}.}
  \textbf{(a)}~Anatomy of a head's per-position output: the QK circuit produces attention weight $\alpha_{t,j}$; the OV circuit produces $W_O \mathbf{v}_j$. Attention-based methods measure only $\alpha$. \methodnamecap (\methodnameabbrev) measures $\phi = \mathbf{u}_{y_t}^\top (\alpha \cdot W_O \mathbf{v}_j)$, capturing the full pipeline.
  \textbf{(b)}~Consequence for non-literal retrieval: two heads read from ``Eiffel Tower'' to answer ``Yuki.'' Head~A attends strongly ($\alpha{=}0.30$), but its OV output is orthogonal to the answer direction ($\phi {\approx} 0$). Head~B attends moderately ($\alpha{=}0.08$) but writes toward the answer ($\phi{=}1.3$). Attention-based methods select Head~A; \methodnameabbrev selects Head~B.}
  \label{fig:overview}
\end{figure}

We introduce \methodnameabbrev, which scores each head by what it \emph{writes} toward the answer rather than where it allocates attention, because non-literal retrieval transforms attended content through the OV circuit before it becomes an answer.
We define a three-step procedure:

\looseness-1
\textbf{Per-Position Logit Contribution.}
Consider a probing trial $\tau$ with needle span $[s_\tau, e_\tau)$ embedded in a context of $N_\tau$ tokens.
Let $\mathcal{A}^\tau$ denote the set of decoding steps at which the model generates a correct answer token (identified by matching against the tokenized gold answer), $y_t \in \mathcal{V}$ the correct token at step $t$, and $N_t$ the total number of key positions available at step $t$.
The contribution of source position $j$ through head $(l,h)$ to the logit of $y_t$ is:
\begin{equation}\label{eq:logit-contribution}
    \phi^{(l,h)}_{t,j} \;=\; \mathbf{u}_{y_t}^\top\, \mathbf{o}^{(l,h)}_{t,j} \;=\; \alpha^{(l,h)}_{t,j} \cdot \mathbf{u}_{y_t}^\top W_O^{(l,h)}\, \mathbf{v}^{(l,h)}_{t,j} \;\in\; \mathbb{R}\,.
\end{equation}
The scalar $\phi^{(l,h)}_{t,j}$ depends on both \emph{where} the head reads (via $\alpha^{(l,h)}_{t,j}$) and \emph{what} it extracts (via $W_O^{(l,h)} \mathbf{v}^{(l,h)}_{t,j}$).
A head that attends strongly to a position whose OV output is orthogonal to $\mathbf{u}_{y_t}$ receives $\phi \approx 0$ despite high attention; conversely, a head performing non-literal retrieval (\eg attending to ``Paris'' to produce ``France'') receives large $\phi$ because its OV circuit transforms the attended representation into an answer-aligned output.
%

%
\looseness-1
\textbf{Spatial Contrast.}
For each head $(l,h)$ at answer step $t$ in trial $\tau$, we define the logit contribution from needle and off-needle positions:
\begin{equation}\label{eq:needle-logit}
    \Phi^{(l,h),+}_{t} = \sum_{j=s_\tau}^{e_\tau-1} \phi^{(l,h)}_{t,j}\,,
    \qquad
    \Phi^{(l,h),-}_{t} = \tfrac{e_\tau - s_\tau}{N_t - (e_\tau - s_\tau)} \sum_{j \notin [s_\tau,\, e_\tau)} \phi^{(l,h)}_{t,j}\,,
\end{equation}
where the rescaling factor $(e_\tau{-}s_\tau)/(N_t{-}(e_\tau{-}s_\tau))$ makes $\Phi^{(l,h),+}_{t}$ and $\Phi^{(l,h),-}_{t}$ comparable: both represent the logit contribution of a region of length $(e_\tau{-}s_\tau)$.
The contrast $\Phi^{(l,h),+}_{t} - \Phi^{(l,h),-}_{t}$ is \emph{spatial}: it compares needle and off-needle positions within a single decoding step, rather than the \emph{temporal} contrast (answer vs.\ non-answer steps) used by attention-based methods.
Spatial contrast yields a score from a single answer step, identifies heads that attend to the needle persistently but write answer-relevant content only from needle positions, and cancels uniform contributors such as token-frequency priors whose $\phi^{(l,h)}_{t,j}$ is large but position-independent.

\looseness-1
\textbf{Aggregation.}
We pool over all answer steps across all trials passing a correctness filter (ROUGE-1 recall $> \rho$, default $\rho = 0.5$), without per-trial normalization.
Let $\mathcal{D}_{\mathrm{pass}} \subseteq \{1, \ldots, T\}$ denote the set of passing trials.
The final score is:
\begin{equation}\label{eq:global-aggregate}
    S_{l,h} = \frac{1}{\sum_{\tau \in \mathcal{D}_{\mathrm{pass}}} |\mathcal{A}^\tau|}
    \sum_{\tau \in \mathcal{D}_{\mathrm{pass}}} \sum_{t \in \mathcal{A}^\tau}
    \Big(\Phi^{(l,h),+}_{t} - \Phi^{(l,h),-}_{t}\Big)\,.
\end{equation}
The score $S_{l,h}$ is the mean over all (trial, answer-step) pairs, with each answer step weighted equally.
%
%
Unlike attention-based methods, we do not clamp $S_{l,h}$ at zero: negative values indicate heads whose logit contribution toward the correct answer originates predominantly from off-needle positions, which points to factors other than needle-specific retrieval (\eg parametric or contextual associations).
A worked example with Per-trial consistency diagnostics is reported in \cref{sec:worked-examples}.
\methodnameabbrev recovers attention-based scoring as a special case when the OV circuit contributes no position-dependent signal.\footnote{\Cref{app:relationship-attention} works through the reduction.}

\section{Experiments}
\label{sec:experiments}

\subsection{Experimental Setup}
\label{sec:setup}

\textbf{Probing benchmark.}
NoLiMa~\citep{modarressi2025nolima} is a \textit{needle-in-a-haystack} benchmark in which each trial embeds a factual statement in a long context and poses a question whose answer requires understanding the needle's meaning rather than copying a literal token from it.
Its non-literal retrieval property expose weaknesses of attention-based scoring.
We use the \texttt{onehop} reasoning, 10 context lengths $\times$ 10 insertion depths, and 3 characters per entry (context range 1{,}000--5{,}000 tokens).
We evaluate \methodnameabbrev on six models from three families: Qwen3 (8B, 14B, 32B)~\citep{qwen3technicalreport}, Gemma-3 (12B, 27B)~\citep{gemma_2025}, and OLMo-3.1 (32B)~\citep{olmo2025olmo3}.
Trials with ROUGE-1 recall $> 0.5$ against the gold answer are retained, following \citet{wu2025retrieval}.

\looseness-1
\textbf{Causal validation via mean-ablation.}
%
For each selected head $(l,h)$ and at every decode step $t$, we replace the post-Q-projection, pre-RoPE query vector with a head-specific calibration vector $\bar{\mathbf{q}}^{(l,h)}$ (computed once over a 50-trial sample; see \cref{sec:additional-details}); we then evaluate ROUGE-L on a disjoint held-out set of 800 NoLiMa trials.
%
%
Mean replacement keeps downstream-layer activations in-distribution~\citep{nanda2023progress,wang2023interpretability}.
%
Architecture-specific implementation details are given in \cref{sec:architecture-adaptations}.

\textbf{Baselines.}
We compare \methodnameabbrev against:
\begin{enumerate*}[label=(\roman*)]
  \item \textbf{Random}: uniformly sampled heads, used to check that non-trivial ablation effects require targeted selection.
  \item \textbf{Wu/NIAH-scored}: the token-matching retrieval score of \citet{wu2025retrieval} computed on NIAH dataset; and
  \item \textbf{Wu/NoLiMa-scored}: the same token-matching criterion computed on NoLiMa probing trials;
\end{enumerate*}
The suffix on \emph{Wu/X-scored} denotes the dataset used to detect the retrieval heads.

\subsection{Ablation Comparison Across Scoring Methods}
\label{sec:ablation-comparison}

\begin{figure}[h]
  \centering
  \includegraphics[width=\linewidth]{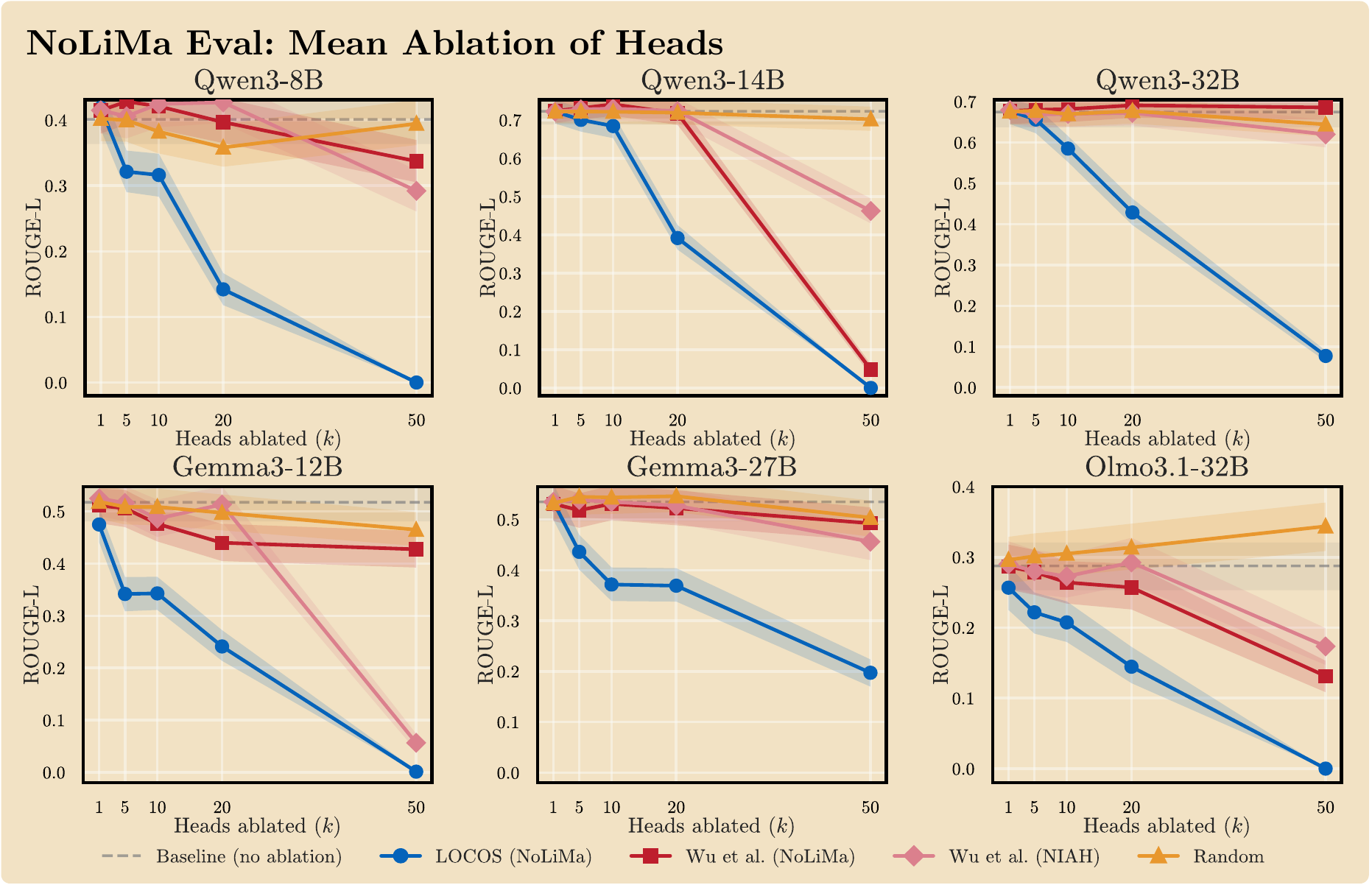}
  \caption{\textbf{\methodnameabbrev heads produce steeper ROUGE-L degradation under mean-ablation across all six models.}
  Each panel shows NoLiMa ROUGE-L (800 trials) as a function of the number of ablated heads $k$ for four scoring methods across three model families at two scales each: Qwen3 (8B, 14B, 32B), OLMo-3.1 (32B), and Gemma-3 (12B, 27B).
  \methodnameabbrev (blue) produces the steepest degradation curve in every model, reaching near-zero ROUGE-L by $k{=}50$ in five of six configurations and severe degradation (${\approx}\,0.1$) in Qwen3-32B.}
  \label{fig:ablation-comparison}
\end{figure}


\Cref{fig:ablation-comparison} compares the four methods under mean-ablation of the top-$k$ heads.
%
%
On Qwen3-8B, ablating the top-$5$ \methodnameabbrev heads already reduces ROUGE-L to $0.321$ (baseline $= 0.401$), while ablating Wu/NIAH heads remains at $0.406$.
By $k{=}50$, \methodnameabbrev reaches $0.000$, whereas Wu/NIAH-scored still achieves $0.292$ and Wu/NoLiMa-scored $0.337$.
Random-head ablation remains near baseline throughout ($0.358$--$0.402$); the degradation is therefore specific to head selection.
Wu/NoLiMa-scored-ranked heads ablation raises ROUGE-L above baseline at intermediate~$k$ ($0.428$ at $k{=}5$), which indicates that its token-matching criterion selects causally irrelevant heads.
\newtakeaway{\methodnameabbrev heads produce the steepest ROUGE-L ablation curve on NoLiMa, reaching near-zero output at lower $k$ than all evaluated baselines.}

\subsection{\revision{Isolating the OV Contribution: An Attention-Only Spatial-Contrast Control}}
\label{sec:ov-isolation}

\revision{
\looseness-1
\methodnameabbrev differs from \citet{wu2025retrieval} in both the per-position observable (OV projection $\phi$ \versus attention-based token matching) and aggregation (spatial \versus temporal contrast).
To isolate the OV contribution, we substitute $\alpha$ for $\phi$ in \Cref{eq:needle-logit,eq:global-aggregate}, yielding an \emph{attention-only spatial-contrast} score $S^{\text{att}}_{l,h}$ that matches \methodnameabbrev's aggregation but removes the OV projection (\cref{prop:reduction,app:rel:reduction}).
}

\begin{figure}[t]
\centering
\includegraphics[width=\linewidth]{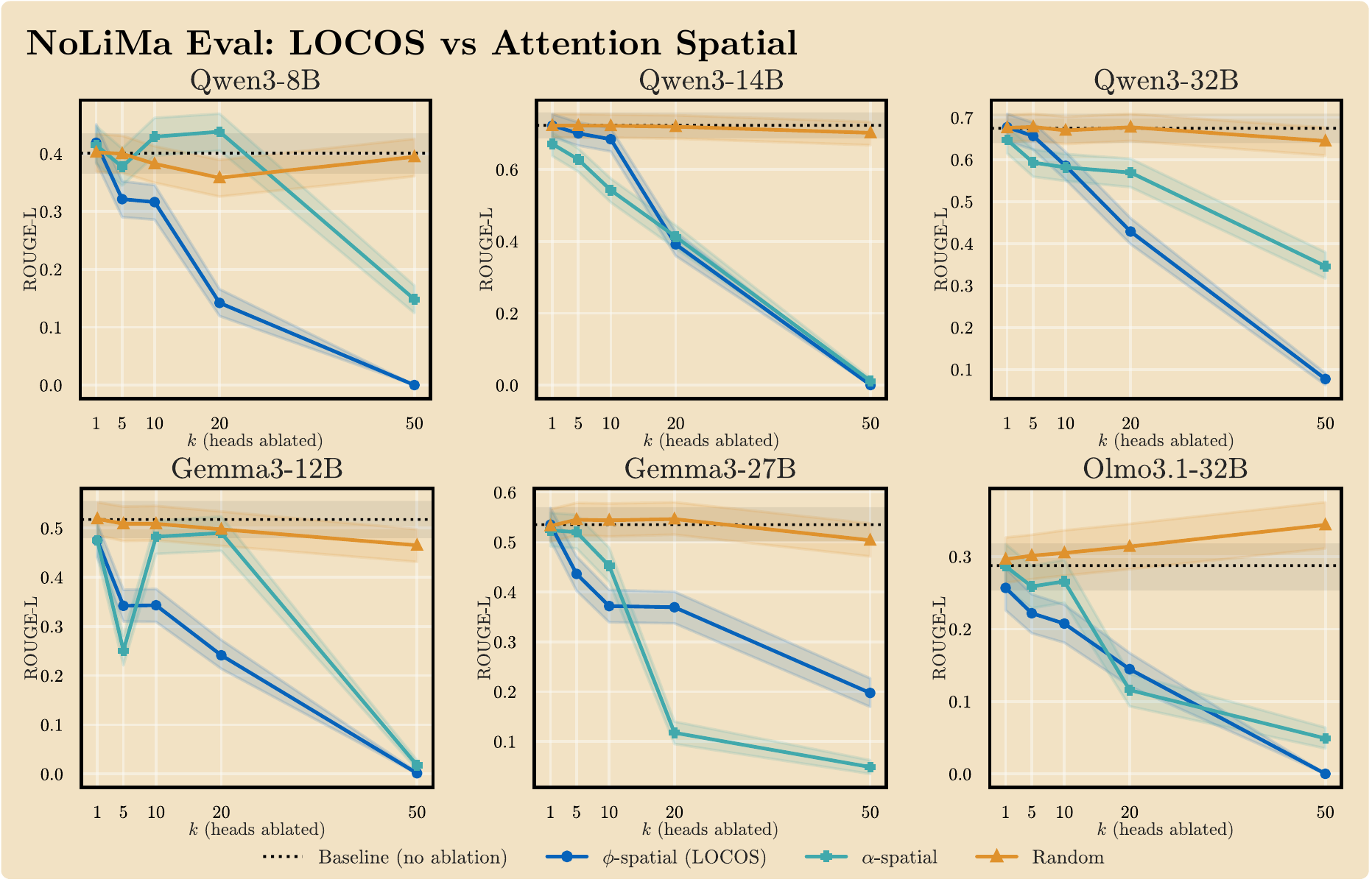}
\caption{
\revision{
\textbf{OV projections improve causal head selection on most models.}
Each panel shows NoLiMa ROUGE-L (800 held-out trials) under mean-ablation of the top-$k$ heads ranked by \methodnameabbrev (blue) and the attention-only control (cyan).
Both scorers use identical spatial-contrast aggregation; only the per-position observable differs.
\methodnameabbrev is stronger on Qwen3-8B, Qwen3-32B, and Gemma-3-12B, comparable on Qwen3-14B and OLMo-3.1-32B, and weaker at large $k$ on Gemma-3-27B.
}
}
\label{fig:ov-isolation}
\end{figure}

\revision{
\Cref{fig:ov-isolation} shows that \methodnameabbrev is more damaging on Qwen3-8B, Qwen3-32B, and Gemma-3-12B: at $k{=}50$, it reaches $0.000$ versus $0.148$ on Qwen3-8B and $0.077$ versus $0.346$ on Qwen3-32B.
The scorers are comparable on Qwen3-14B and OLMo-3.1-32B, while Gemma-3-27B inverts at depth: the attention-only control is more damaging at $k{\in}\{20,50\}$ ($0.118$ vs.\ $0.369$ at $k{=}20$).\footnote{We analyze the Gemma-3-27B inversion in \cref{app:direct-path:gemma-inversion}, where tuned-lens-corrected projections test whether the direct-path assumption explains this inversion.}
Thus, OV projection primarily improves reliability rather than uniformly increasing effect size: \methodnameabbrev is the only scorer in our evaluation that produces severe or total collapse across configurations.
}

\newtakeaway{
\revision{
Under identical spatial-contrast aggregation, an attention-only control is competitive on average but substantially less reliable across models.
The OV projection enables \methodnameabbrev to consistently identify highly causal heads across all six configurations.
}
}

\subsection{Bottom-$k$ Control: Spatial Source Matters}
\label{sec:bottom-k}


A potential objection to the ablation result is circularity: the method selects heads whose OV output projects onto $\mathbf{u}_{y_t}$, so ablating them mechanically reduces the $y_t$ logit regardless of whether the heads perform retrieval.
If this were the case, \emph{any} set of heads with large answer-aligned logit contribution should be equally causal when ablated, \emph{irrespective of whether that contribution originates from the needle or from unrelated context}.
Heads with the most negative spatial contrast scores ($S_{l,h} \ll 0$) (\ie \emph{bottom-$k$ heads}) provide a direct test: they contribute strongly to the correct-answer logit but predominantly from off-needle positions.
The absolute logit contribution of these bottom-$k$ heads is large, so the comparison with top-$k$ heads is not confounded by magnitude.
The full score distribution (\cref{app:score-distribution}) confirms that all bottom-50 heads have strictly negative $S_{l,h}$ in every model, so the bottom-$k$ experiments exclusively target heads with off-needle-dominant contributions.

\begin{figure}[t]
  \centering
  \includegraphics[width=\linewidth]{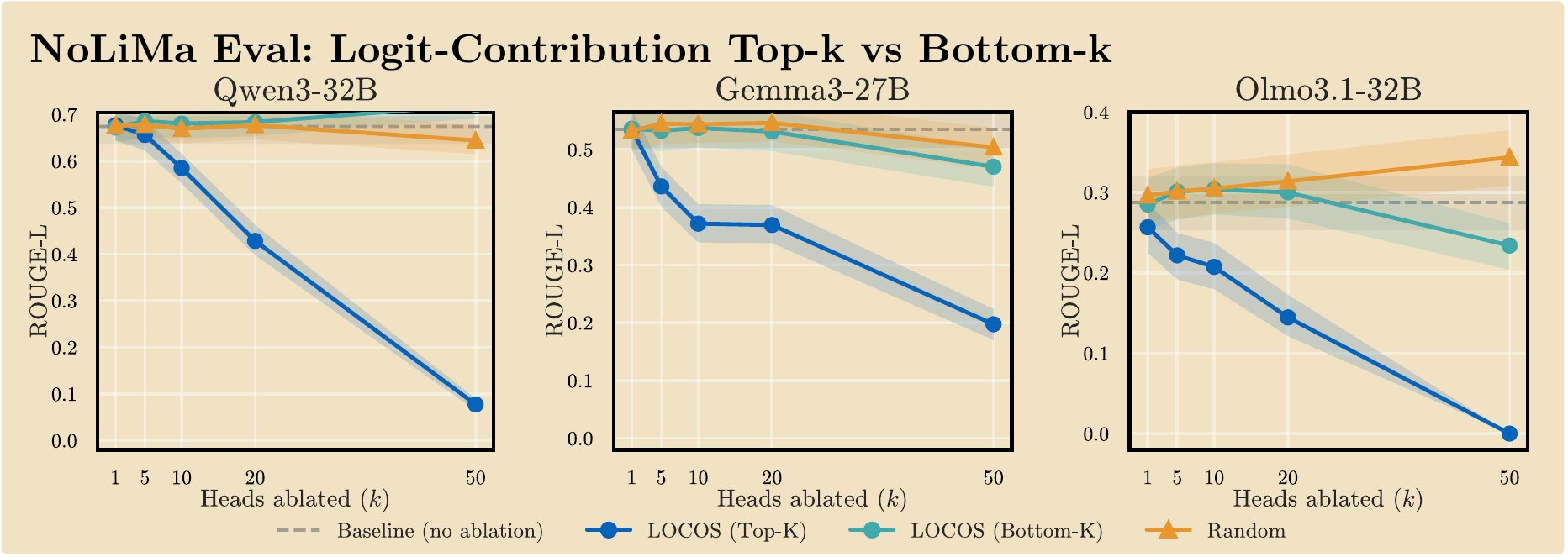}
  \caption{\textbf{Bottom-$k$ ablation does not degrade retrieval.}
  Each panel shows NoLiMa ROUGE-L as a function of ablation depth~$k$ for top-$k$ (blue), bottom-$k$ (cyan), and random heads (orange) for three representative models (one per family); the full six-model version is in \cref{app:six-model-figures}.
  Top-$k$ heads produce steep degradation; bottom-$k$ heads track the random baseline despite having equally large absolute logit contribution, ruling out the circularity objection.}
  \label{fig:ablation-bottomk}
\end{figure}

As shown in \cref{fig:ablation-bottomk}, mean-ablating the top-$k$ \methodnameabbrev heads produces steep ROUGE-L degradation, reaching near-zero at $k{=}50$ in most models.
By contrast, ablating the bottom-$k$ heads leaves ROUGE-L near baseline, tracking the random-head control at all ablation depths.

\newtakeaway{
\looseness-1
Ablating bottom-$k$ heads (those contributing to the answer logit from off-needle positions) does not degrade NoLiMa ROUGE-L. This indicates that the ablation result (\cref{sec:ablation-comparison}) is not merely an artifact of removing answer-aligned signal.
}

\subsection{Layer Distribution}
\label{sec:layer-distribution}

\begin{figure}[h]
  \centering
  \includegraphics[width=\linewidth]{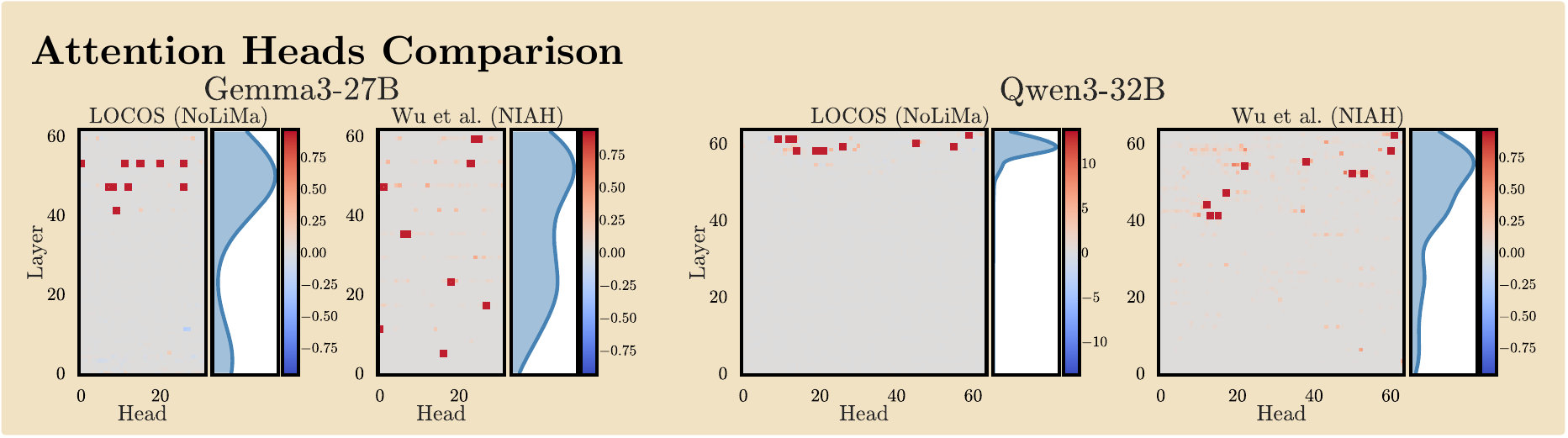}
  \caption{\textbf{\methodnameabbrev heads are more concentrated in late layers than Wu/NIAH-scored scores.}
  Layer $\times$ Head heatmaps on NoLiMa for Gemma-3-27B (left) and Qwen3-32B (right).
  The left-hand panel of each model shows \methodnameabbrev; the right shows Wu/NIAH-scored token-matching.
  Red squares mark top-10 heads.
  Both \methodnameabbrev and Wu/NIAH-scored assign high scores predominantly to late layers, but Wu/NIAH-scored additionally identifies heads in early-to-middle layers.
  }
  \label{fig:heatmap}
\end{figure}

\Cref{fig:heatmap} visualizes score distributions across layers for Gemma-3-27B and Qwen3-32B.
\revision{\methodnameabbrev is most concentrated in the Qwen3 family and Gemma-3-27B; at KV-group granularity, Gemma-3-12B and OLMo-3.1-32B span broader layer ranges (\cref{app:kv-group-layer}).}
One possible explanation is that \methodnameabbrev relies on a \emph{direct-path assumption} that is more accurate near the output, biasing it toward late layers.

A tuned-lens check on Gemma-3-27B preserves the late-layer band, and causal activation patching on Qwen3-8B and Gemma-3-12B also concentrates top-10 heads in upper layers (\cref{app:direct-path:empirical,app:direct-path:beyond}).
The top-10 sets overlap only marginally between the two scores (2/10 on Qwen3-8B, 3/10 on Gemma-3-12B); the top-$k$ \methodnameabbrev set should therefore be read as a collectively causal retrieval circuit, validated by the group ablations of \cref{sec:ablation-comparison}, rather than a list of individually load-bearing heads.

\newtakeaway{\revision{\methodnameabbrev scores concentrate in late layers in the Qwen3 family and Gemma-3-27B; the layer pattern is family-dependent (\cref{app:kv-group-layer}).} A tuned-lens variant (\cref{app:direct-path:empirical}) and a causal-attribution probe (\cref{app:direct-path:beyond}) confirm the late-layer concentration is not a direct-path artifact \revision{on the models examined}; the top-$k$ \methodnameabbrev set should be read as a collectively causal retrieval circuit \revision{rather than a universal architectural pattern across all families}.}

\subsection{Retrieval Specificity}
\label{sec:specificity}


\looseness-1
A potential concern is that \methodnameabbrev identifies generically important heads for model output, not retrieval-specific.
%
%
We test this by ablating \methodnameabbrev heads and measuring performance on tasks that do not require retrieval:
\begin{enumerate*}[label=(\roman*)]
    \item \textbf{City--country associations}: parametric factual recall (\eg ``Which country does Paris belong to?'').
    \item \textbf{PopQA}~\citep{mallen2023popqa}: top 100 popular QA drawn from a knowledge base (\eg ``Who is the mother of Jesus?'').
    \item \textbf{Arithmetic}: two-operand addition and subtraction (\eg ``What is 47 + 23?'').\footnote{Full dataset available at \href{https://huggingface.co/datasets/aryopg/parametric-arithmetic-eval}{\raisebox{-0.2\height} {\includegraphics[height=1em]{hf-logo.png}} \texttt{aryopg/parametric-arithmetic-eval}}}
\end{enumerate*}

\looseness-1
To compare specificity across methods, we define the \emph{Dissociation Score} at ablation depth $k$ as $\mathrm{DS}(k) = \Delta R(k) - \Delta P(k)$, where $\Delta R(k) = (R_0 - R(k))/R_0$ and $\Delta P(k) = (P_0 - P(k))/P_0$ are the relative drops in NoLiMa ROUGE-L and aggregate parametric accuracy (mean of city--country, PopQA, and arithmetic), respectively.
$R_0, P_0$ are baselines.
$\mathrm{DS}{=}1$ means retrieval is fully destroyed with zero parametric damage; $k^{*} = \arg\max_k \mathrm{DS}(k)$ identifies the most-specific ablation depth.

\begin{figure}[h]
\centering
\includegraphics[width=\linewidth]{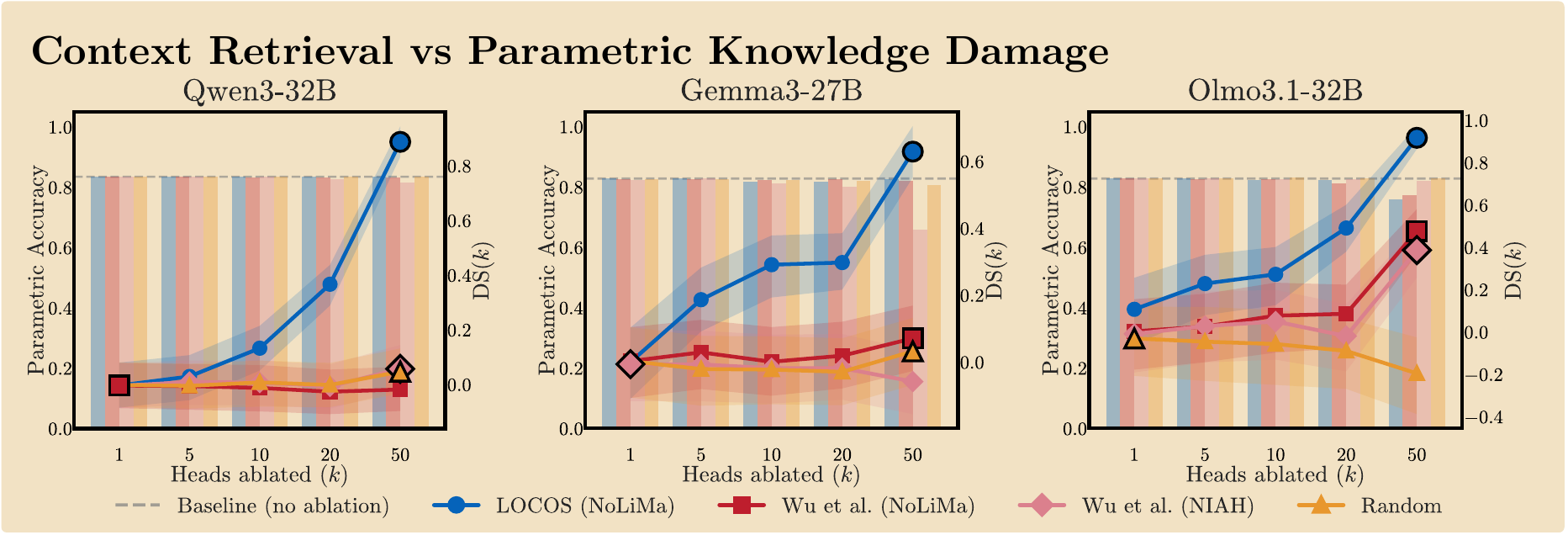}
\caption{
\looseness-1
\textbf{\methodnameabbrev heads exhibit the strongest functional dissociation between retrieval and parametric capabilities.}
Each panel shows DS$(k)$ (lines, right axis) and parametric accuracy (bars, left axis) as a function of ablation depth $k$ for four scoring methods, on three representative models (one per family); the full six-model version is in \cref{app:six-model-figures}.
Higher DS indicates that ablation degrades retrieval far more than parametric tasks.
\methodnameabbrev (blue) achieves the highest DS in every model configuration; the enlarged marker indicates $k^*$, the point of maximum dissociation.
}
\label{fig:fds}
\end{figure}

%
\Cref{fig:fds} shows DS$(k)$ for each scoring method, with \methodnameabbrev achieving the highest peak in every model.
Parametric accuracy likewise remains stable under bottom-$k$ ablation (\cref{app:bottomk-ds}); these heads are therefore not responsible for parametric recall either.

\newtakeaway{Retrieval heads identified by \methodnameabbrev are retrieval-specific: ablating them severely degrades contextual retrieval while leaving parametric accuracy near baseline (\cref{fig:fds}), and \methodnameabbrev achieves the highest dissociation-score peak in every configuration.}

\subsection{Literal vs.\ Non-Literal Retrieval Specificity}
\label{sec:literal-vs-nonliteral}


\Cref{sec:specificity} established that ablating \methodnameabbrev heads degrades non-literal retrieval far more than parametric tasks.
\revision{A write-aware detector that measures the full OV pipeline should, in principle, identify retrieval heads in both literal and non-literal regimes; the novel contribution over attention-based scoring is access to the non-literal subset.}
We test this directly by mean-ablating the same top-$k$ \methodnameabbrev heads (selected on NoLiMa probing trials) and evaluating on both NoLiMa and standard NIAH~\citep{kamradt2023niah}\revision{, treating the two benchmarks as probes for the non-literal and literal retrieval circuits respectively}, using the protocol of \cref{sec:setup} for both benchmarks.

\begin{figure}[h]
\centering
\includegraphics[width=\linewidth]{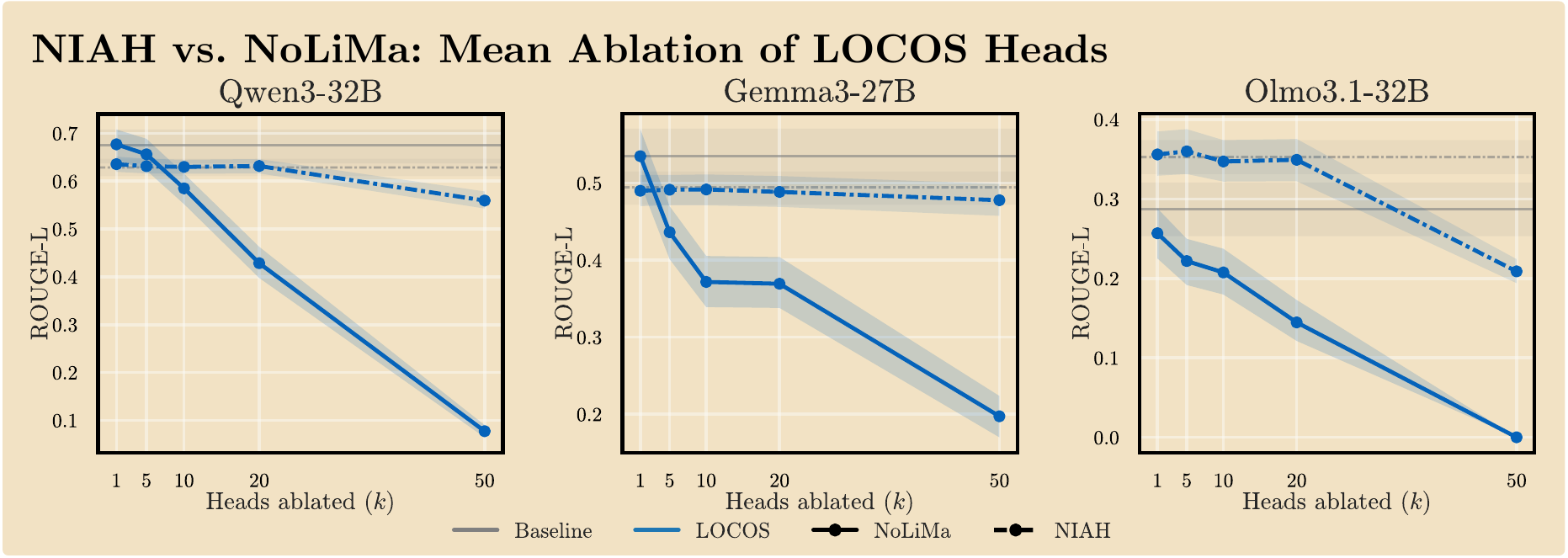}
\caption{
  \textbf{Ablating \methodnameabbrev heads damages non-literal retrieval more than literal retrieval.}
  Each panel shows ROUGE-L on NoLiMa (solid) and standard NIAH (dashed) under mean-ablation of the same top-$k$ \methodnameabbrev heads, with the NoLiMa and NIAH baselines marked by solid and dashed gray lines.
  Three representative models are shown here; the full six-model version is in \cref{app:six-model-figures}.
  The NoLiMa curve declines more steeply in every configuration, reaching near-zero at $k{=}50$ in five of six models, and the NoLiMa--NIAH gap widens with $k$.
}
\label{fig:ablation-niah-vs-nolima}
\end{figure}

\Cref{fig:ablation-niah-vs-nolima} shows that the non-literal curve drops faster than the literal curve in every configuration, but the literal curve also drops in five of six models, so the top-$k$ \methodnameabbrev set is not strictly non-literal-specific.
Both retrieval regimes require heads that read the answer span; they differ only in what the OV circuit writes.
The top-$k$ set contains both kinds of heads: those whose OV write is answer-aligned in both regimes damage both benchmarks when ablated, while those aligned only in the non-literal regime explain the larger NoLiMa drop.
\methodnameabbrev therefore identifies retrieval-relevant heads in both regimes; its novelty relative to attention-based scoring lies in identifying the non-literal subset.

\newtakeaway{Ablating top-$k$ \methodnameabbrev heads degrades both non-literal (NoLiMa) and literal (NIAH) retrieval, with a steeper drop on NoLiMa. The dissociation indicates that \methodnameabbrev captures retrieval-relevant heads beyond those identified by attention-based scoring, including heads that contribute to non-literal retrieval which prior detectors miss.}

\subsection{Downstream Long-Context Evaluation}
\label{sec:downstream}


We test transfer by re-running the head ablation on two downstream long-context benchmarks that were not used to score heads.
\textbf{MuSiQue}~\citep{trivedi2022musique} is a multi-hop QA benchmark whose questions require composing facts across multiple paragraphs of a long supporting context.
\textbf{BABILong} (qa2 and qa3 subsets)~\citep{kuratov2024babilong} requires tracing an object's trajectory through movements interleaved with distractor narrative; literal token copying is insufficient because each candidate location is mentioned many times for many entities (see \cref{app:downstream-example} for an example).
We fix $k{=}50$ for direct comparison with the headline NoLiMa result and retain the random-heads control and the Wu/NIAH-scored baseline.

\begin{figure}[h]
  \centering
  \includegraphics[width=\linewidth]{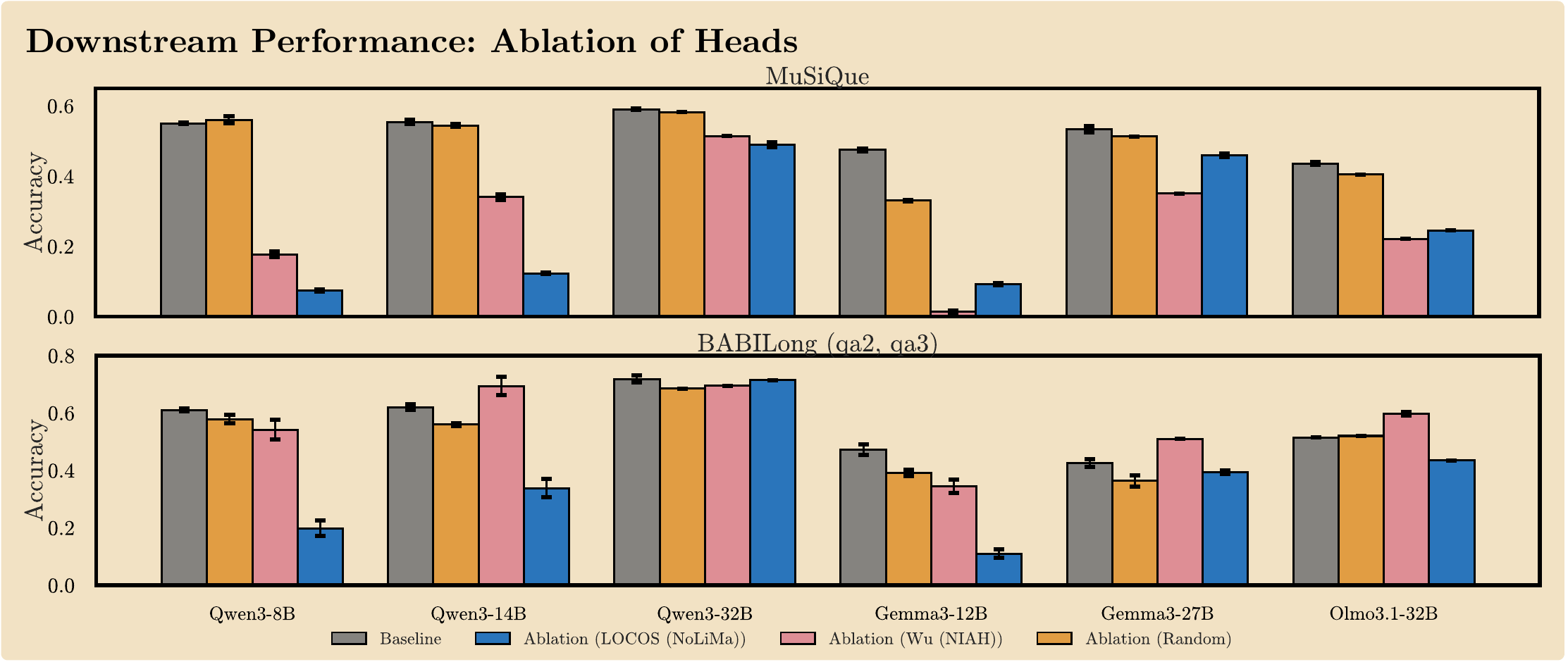}
  \caption{\textbf{Mean-ablating top-$50$ \methodnameabbrev heads degrades downstream long-context performance, most strongly on the Qwen3 family.}
  Accuracy on MuSiQue (top) and BABILong qa2+qa3 (bottom) for six models. Bars show the unablated baseline (gray) and the three ablation conditions: random heads (orange), Wu/NIAH-scored heads (pink), and \methodnameabbrev (blue). Error bars are standard deviations across three independent runs.
  \methodnameabbrev produces the largest drop in $6$ of $12$ model--benchmark cells; Wu/NIAH-scored ablation is more damaging on MuSiQue for Gemma-3 and OLMo-3.1, and slightly raises BABILong accuracy on Qwen3-14B, Gemma-3-27B, and OLMo-3.1-32B.}
  \label{fig:downstream}
\end{figure}

\Cref{fig:downstream} shows that ablating the top-$50$ \methodnameabbrev heads on Qwen3-8B drops MuSiQue accuracy from $0.55$ to $0.08$ and BABILong from $0.62$ to $0.20$.
Across the twelve model--benchmark cells, \methodnameabbrev is the most damaging ablation in six and produces a drop of at least $0.10$ below baseline in eight, while the random-heads control stays within $0.05$ of baseline in eleven; the degradation is therefore specific to the heads \methodnameabbrev identifies.
The ranking against attention-based scoring is benchmark-dependent: on MuSiQue, Wu/NIAH-scored ablation is more damaging on Gemma-3 (12B, 27B) and OLMo-3.1-32B, consistent with multi-hop QA recruiting both literal and non-literal retrieval; on BABILong, Wu/NIAH-scored ablation slightly \emph{raises} accuracy on Qwen3-14B, Gemma-3-27B, and OLMo-3.1-32B, while \methodnameabbrev ablation drops accuracy in every model.

\newtakeaway{Ablating top-$50$ \methodnameabbrev heads degrades downstream long-context performance, most strongly on the Qwen3 family (Qwen3-8B MuSiQue: $0.55{\to}0.08$, BABILong: $0.62{\to}0.20$)}

\section{Related Work}
\label{sec:related-work}

\looseness-1
\textbf{Attention is not explanation.}
%
The faithfulness of attention weights as an explanation has been contested \citep{jain2019attention, serrano2019attention, wiegreffe2019attention}, with \citet{bastings2020elephant} arguing that gradient-based saliency methods provide more reliable attribution than attention alone.
Our method, \methodnameabbrev, measures the head's direct write onto the answer logit rather than its attention distribution, aligning with the \emph{task-grounded faithfulness} criteria~\citep{wiegreffe2019attention}.

\textbf{OV/QK circuit analysis and attribution.}
Our method builds on the QK/OV circuit decomposition of \citet{elhage2021mathematical} and generalizes the literal copying mechanism of induction heads~\citep{olsson2022inductionheads} to non-literal retrieval.
The logit lens~\citep{nostalgebraist2020logitlens} and tuned lens~\citep{belrose2025tunedlens} project residual-stream states onto the unembedding matrix at per-layer granularity; \methodnameabbrev applies an analogous projection at per-head, per-position granularity with spatial contrast, building on direct logit attribution.
%
\citet{mcdougall2023copy} characterize copy suppression heads whose OV circuits actively inhibit direct token copying, which may be a complementary mechanism to the non-literal retrieval heads identified here.

\looseness-1
\textbf{Retrieval head identification.}
Existing methods identify retrieval heads through attention-weight observables.
\citet{wu2025retrieval} define retrieval heads via token matching on needle-in-a-haystack tasks; \citet{fu2025headkv}, \citet{lin2025compresskv}, and \citet{xiao2025duoattention} extend attention-based head scoring to KV cache allocation, semantic attention-mass criteria, and retrieval/streaming-head separation.
The broader KV cache compression literature~\citep{zhang2023h2o,li2024snapkv,xiao2024streamingllm,cai2024pyramidkv} motivates per-head retrieval scoring by showing that uniform token eviction misses structured head and layer importance.
These methods measure the QK circuit, but not the OV circuit.
%
%
For literal retrieval, this suffices; for non-literal retrieval, the OV circuit transforms attended content, and attention-based methods miss these heads.
\citet{sun2025redeep} analyze OV circuits for hallucination detection in RAG, decomposing contributions into copying heads (attention) and knowledge FFNs; our method differs by scoring individual heads at per-position granularity with spatial contrast, targeting retrieval head identification.

\section{Conclusion}
\label{sec:conclusion}

\looseness-1
Existing retrieval-head detectors~\citep{wu2025retrieval,fu2025headkv,lin2025compresskv} score attention heads by where they read.
This identifies the heads that copy literal tokens, but misses the heads that synthesize answers from the meaning of an attended span.
We presented \methodnameabbrev (\methodnametitle), a detector that scores each head by the projection of its OV-circuit output onto the answer-token unembedding direction, and uses a needle-versus-off-needle spatial contrast to isolate position-specific writes within a single forward pass per probing trial.
Across the six configurations evaluated, mean-ablating the top-$k$ heads selected by \methodnameabbrev collapses NoLiMa ROUGE-L at lower $k$ than every attention-based baseline we evaluate.
%
%
The selected heads are retrieval-specific: under the same ablation, parametric recall and arithmetic reasoning remain at baseline.
\revision{The top-$k$ \methodnameabbrev set captures retrieval-relevant heads in both literal and non-literal regimes; its advantage over attention-based scoring is in additionally identifying the non-literal subset that the token-matching criterion misses (\cref{sec:literal-vs-nonliteral}).}

\textbf{Limitations.}
\label{sec:limitations}
\begin{enumerate*}[label=(\roman*)]
    \item \textbf{Off-needle baseline:} If the context contains distractor information related to the answer, $\Phi^-$ rises and scores drop for heads performing broad semantic matching rather than targeted needle retrieval---desirable for span-specific retrieval, but may miss heads performing more diffuse contextual integration;
    \item \textbf{Architecture coverage}: Our model selection does not include Mixture-of-experts routing, encoder--decoder stacks, and state-space hybrids. The per-head OV decomposition still applies in principle, but the causal head-ablation magnitudes and late-layer concentration we report should not be assumed to transfer without verification.
\end{enumerate*}

\section*{Acknowledgements}
\label{sec:acknowledgements}

APG was supported by the United Kingdom Research and Innovation (grant EP/S02431X/1), UKRI Centre for Doctoral Training in Biomedical AI at the University of Edinburgh, School of Informatics.
BA is a Fellow at and has been supported by the Generative AI Lab (GAIL) at the University of Edinburgh.
PM was supported by the Engineering and Physical Sciences Research Council (EPSRC) through the AI Hub in Generative Models (grant number EP/Y028805/1).
This work was supported by the Edinburgh International Data Facility (EIDF) and the Data-Driven Innovation Programme at the University of Edinburgh.
We also thank Neel Rajani, Raj Bhalwankar, Matteo Attimonelli, and Federico Tiblias for their helpful comments and suggestions.

\bibliography{main}
\bibliographystyle{plainnat}

\newpage

\appendix

\textbf{\LARGE Appendix}
\vskip 4mm

\text{\LARGE{Table of Contents}}
\vskip 4mm
\hrule height .5pt
\vskip 4mm

\textit{Reproducibility.}
\begin{itemize}[label={},leftmargin=*]
    \item \textbf{\textcolor{black}{\hyperref[app:assets]{Appendix A - Datasets, Models, and Licenses}}} \dotfill \pageref{app:assets}
    \item \textbf{\textcolor{black}{\hyperref[sec:additional-details]{Appendix B - Experimental Setup Details}}} \dotfill \pageref{sec:additional-details}
    \item \textbf{\textcolor{black}{\hyperref[app:compute]{Appendix C - Compute Resources}}} \dotfill \pageref{app:compute}
    \item \textbf{\textcolor{black}{\hyperref[app:code-data]{Appendix D - Code and Data Availability}}} \dotfill \pageref{app:code-data}
    \item \textbf{\textcolor{black}{\hyperref[app:broader-impacts]{Appendix E - Broader Impacts}}} \dotfill \pageref{app:broader-impacts}
    \item \textbf{\textcolor{black}{\hyperref[app:llm-usage]{Appendix F - Declaration of LLM Usage}}} \dotfill \pageref{app:llm-usage}
\end{itemize}

\textit{Method walk-through.}
\begin{itemize}[label={},leftmargin=*]
    \item \textbf{\textcolor{black}{\hyperref[sec:worked-examples]{Appendix G - Worked Example}}} \dotfill \pageref{sec:worked-examples}
    \item \textbf{\textcolor{black}{\hyperref[sec:architecture-adaptations]{Appendix H - Architecture Adaptations}}} \dotfill \pageref{sec:architecture-adaptations}
\end{itemize}

\textit{Additional empirical analyses.}
\begin{itemize}[label={},leftmargin=*]
    \item \textbf{\textcolor{black}{\hyperref[app:score-distribution]{Appendix I - Logit-Contribution Score Distribution}}} \dotfill \pageref{app:score-distribution}
    \item \textbf{\textcolor{black}{\hyperref[app:bottomk-ds]{Appendix J - Bottom-$k$ Dissociation Score}}} \dotfill \pageref{app:bottomk-ds}
    \item \textbf{\textcolor{black}{\hyperref[app:kv-group-layer]{Appendix K - KV-Group $\times$ Layer View}}} \dotfill \pageref{app:kv-group-layer}
    \item \textbf{\textcolor{black}{\hyperref[app:six-model-figures]{Appendix L - Six-Model Versions of Main-Text Ablation Figures}}} \dotfill \pageref{app:six-model-figures}
    \item \textbf{\textcolor{black}{\hyperref[app:downstream-example]{Appendix M - Downstream Benchmark Example}}} \dotfill \pageref{app:downstream-example}
\end{itemize}

\textit{Theoretical analysis.}
\begin{itemize}[label={},leftmargin=*]
    \item \textbf{\textcolor{black}{\hyperref[app:direct-path]{Appendix N - Direct-Path Robustness via Tuned Lens}}} \dotfill \pageref{app:direct-path}
    \item \textbf{\textcolor{black}{\hyperref[app:relationship-attention]{Appendix O - Relationship to Attention-Based Scoring}}} \dotfill \pageref{app:relationship-attention}
\end{itemize}

\textit{Outlook.}
\begin{itemize}[label={},leftmargin=*]
    \item \textbf{\textcolor{black}{\hyperref[app:future-work]{Appendix P - Future Work}}} \dotfill \pageref{app:future-work}
\end{itemize}
\vskip 2mm
\hrule height .5pt
\vskip 10mm

\newpage


\section{Datasets, Models, and Licenses}
\label[appendix]{app:assets}

\Cref{tab:assets} lists every external asset used in this work, together with the version we used and its license. All assets are used within the terms of their respective licenses; we do not redistribute model weights or dataset content.

\begin{table}[h]
\centering
\small
\caption{\textbf{External assets used in this work.} Model checkpoints are accessed via the HuggingFace Hub; dataset and library versions are the ones used in our experiments.}
\label{tab:assets}
\resizebox{\linewidth}{!}{%
\begin{tabular}{@{}llll@{}}
\toprule
\textbf{Asset} & \textbf{Identifier / version} & \textbf{Use} & \textbf{License} \\
\midrule
\multicolumn{4}{l}{\textit{Models}} \\
Qwen3-8B   & \texttt{Qwen/Qwen3-8B}   & detection, ablation & Apache-2.0 \\
Qwen3-14B  & \texttt{Qwen/Qwen3-14B}  & detection, ablation & Apache-2.0 \\
Qwen3-32B  & \texttt{Qwen/Qwen3-32B}  & detection, ablation & Apache-2.0 \\
Gemma-3-12B & \texttt{google/gemma-3-12b-it}  & detection, ablation & Gemma Terms of Use\textsuperscript{\dag} \\
Gemma-3-27B & \texttt{google/gemma-3-27b-it}  & detection, ablation & Gemma Terms of Use\textsuperscript{\dag} \\
OLMo-3.1-32B  & \texttt{allenai/OLMo-3.1-32B-Instruct} & detection, ablation & Apache-2.0 \\
\midrule
\multicolumn{4}{l}{\textit{Datasets}} \\
NoLiMa~\citep{modarressi2025nolima} & v1 (\texttt{onehop} reasoning type) & probing benchmark & CC-BY-4.0 \\
NIAH~\citep{kamradt2023niah} & standard configuration & baseline scoring (Wu/NIAH) & MIT \\
PopQA~\citep{mallen2023popqa} & v1 & parametric specificity control & MIT \\
City--country associations & authors' construction & parametric specificity control & released with this paper \\
Arithmetic (two-operand $\pm$) & authors' construction & parametric specificity control & released with this paper \\
\midrule
\multicolumn{4}{l}{\textit{Software}} \\
vLLM & v0.10.0 & model loading, weight sharding & Apache-2.0 \\
PyTorch & 2.4 & numerical backend & BSD-3-Clause \\
Transformers (HuggingFace) & 4.45 & tokenizers, configs & Apache-2.0 \\
\texttt{rouge-score}~\citep{lin2004rouge} & 0.1.2 & ROUGE-L evaluation & Apache-2.0 \\
TransformerLens~\citep{nanda2022transformerlens} & 2.7 & inspiration for direct logit attribution & MIT \\
\bottomrule
\end{tabular}
}
{\footnotesize \textsuperscript{\dag}Gemma Terms of Use permit research and commercial use subject to a prohibited-use policy and an attribution requirement.\par}
\end{table}

\section{Experimental Setup Details}
\label[appendix]{sec:additional-details}

This appendix consolidates the procedural details that allow the experiments in \cref{sec:experiments} to be reproduced. \Cref{tab:hyperparameters} summarizes the hyperparameters used throughout; the paragraphs below give the rationale and the few cases where defaults were varied.

\begin{table}[h]
\centering
\small
\caption{\textbf{Hyperparameters used throughout the paper.}}
\label{tab:hyperparameters}
\resizebox{\linewidth}{!}{%
\begin{tabular}{@{}lll@{}}
\toprule
\textbf{Component} & \textbf{Value} & \textbf{Note} \\
\midrule
\multicolumn{3}{l}{\textit{Probing}} \\
NoLiMa reasoning type & \texttt{onehop} & following \citet{modarressi2025nolima} \\
Context lengths & 10 settings, 1k--5k tokens & standard NoLiMa protocol \\
Insertion depths & 10 per length & standard NoLiMa protocol \\
Characters per entry & 3 & standard NoLiMa protocol \\
Trial filter & ROUGE-1 recall $> 0.5$ & following \citet{wu2025retrieval} \\
\multicolumn{3}{l}{\textit{Decoding (probing and evaluation)}} \\
Strategy & greedy (temperature $= 0$) & deterministic \\
Max new tokens & 50 & sufficient for NoLiMa answers (2--5 tokens) \\
\multicolumn{3}{l}{\textit{Aggregation}} \\
Bootstrap resamples ($B$) & 1{,}000 & for 95\% CI on $S_{l,h}$ \\
Confidence interval & 2.5\textsuperscript{th}--97.5\textsuperscript{th} percentile & non-parametric \\
\multicolumn{3}{l}{\textit{Ablation}} \\
Calibration trials & 50 (sampled from passing pool) & for mean-activation \\
Held-out evaluation trials & 800 & disjoint from probing pool \\
Ablation depths $k$ & $\{0, 5, 10, 20, 30, 40, 50\}$ & for top-$k$ and bottom-$k$ \\
\bottomrule
\end{tabular}
}
\end{table}

\textbf{Calibration for mean-ablation.}
Mean activations for the ablation intervention are computed from 50 NoLiMa trials sampled uniformly from the pool of ROUGE-passing trials.
For each head $(l,h)$, the calibration vector $\bar{\mathbf{q}}^{(l,h)} \in \mathbb{R}^{d_h}$ is a \emph{mean of trial-means}: we first compute the token-average of the post-Q-projection query $\mathbf{q}^{(l,h)}_t = W_Q^{(l,h)}\, \mathbf{x}^{(l)}_t$ within each calibration trial, then average these per-trial vectors across the 50 trials.
Each trial therefore contributes a single equal-weight entry irrespective of its sequence length; a flat mean over all (token, trial) pairs would up-weight longer trials.
During ablation, the hook intercepts the post-Q-projection vector \emph{before} the rotary position embedding is applied and replaces it with $\bar{\mathbf{q}}^{(l,h)}$ at every token position; RoPE then rotates $\bar{\mathbf{q}}^{(l,h)}$ position-by-position in the usual way.
The value projection $W_V^{(l,h)}$ and the output projection $W_O^{(l,h)}$ are unchanged.
Because $\bar{\mathbf{q}}^{(l,h)}$ is a single content-independent vector, the resulting attention logits depend only on the keys (and on the position-dependent RoPE rotation of $\bar{\mathbf{q}}^{(l,h)}$), so the head's attention distribution becomes content-independent (approximately uniform across positions).

\textbf{ROUGE-L evaluation.}
ROUGE-L is computed at the summary level (longest common subsequence between the generated text and gold answer, normalized by gold length) using the \texttt{rouge-score} library~\citep{lin2004rouge}.
Generation uses greedy decoding (temperature $= 0$) with a maximum of 50 new tokens.

\textbf{Bootstrap confidence intervals.}
For each head's global score $S_{l,h}$, we resample the set of ROUGE-passing trials $B = 1{,}000$ times (with replacement), recompute $S_{l,h}$ for each bootstrap sample, and report the 2.5\textsuperscript{th} and 97.5\textsuperscript{th} percentiles as the 95\% confidence interval.

\textbf{Wu et al.\ scoring on NoLiMa and NIAH.}
The Wu et al.\ retrieval score assigns credit at decode step $t$ if (i)~the head's argmax attention position falls within the needle span, and (ii)~the token at that position matches the generated token.
On NoLiMa, this criterion undercounts retrieval heads because of token-identity mismatch: the question is asked non-literally, so the attended needle token rarely matches the generated answer token.
The top head's mean score drops from $0.97$ (NIAH) to $0.03$ (NoLiMa), a $30{\times}$ gap driven by the scoring criterion rather than by differences in retrieval behavior.

\section{Compute Resources}
\label[appendix]{app:compute}

\textbf{Hardware.}
All experiments were run on 2$\times$ NVIDIA H100 80GB.
Tensor parallelism was used for the larger checkpoints: $\mathrm{TP}=2$ for the 14B/12B/27B/32B models, $\mathrm{TP}=1$ for Qwen3-8B.

\textbf{Per-experiment cost.}
Detection requires one forward pass per probing trial, generating tokens autoregressively while recording attention weights and the value cache (the primary memory overhead).
Ablation evaluation is similarly one forward pass per trial per ablation depth $k$, plus the 50-trial calibration pass.

\section{Code and Data Availability}
\label[appendix]{app:code-data}

The detection, ablation, and evaluation code is released at \repo. The repository contains:
\begin{itemize}[nosep,leftmargin=*]
    \item Detection scripts producing the per-head $S_{l,h}$ scores for each model in \cref{sec:setup}.
    \item The vLLM-based ablation driver described in \cref{sec:architecture-adaptations}, including the monkey-patches for each supported attention class.
    \item Evaluation scripts for NoLiMa ROUGE-L (\cref{sec:ablation-comparison}), the parametric controls (\cref{sec:specificity}), and the bottom-$k$ analyses (\cref{sec:bottom-k}).
    \item A \texttt{README} with exact commands, environment specification (\texttt{environment.yml}, pinned package versions matching \cref{tab:assets}), and the random seeds used for all sampling steps.
\end{itemize}
NoLiMa probing inputs are generated from the public NoLiMa release via the included generation script; we do not redistribute the underlying haystack content. The author-released parametric control sets (city--country associations, two-operand arithmetic) are included as JSON files in the repository.

\section{Broader Impacts}
\label[appendix]{app:broader-impacts}

\methodnameabbrev is a diagnostic tool: it identifies attention heads that contribute to non-literal retrieval, without modifying model weights or behavior. Its primary intended uses are interpretability research and downstream applications that benefit from head-level retrieval signal, such as KV cache compression~\citep{zhang2023h2o,li2024snapkv,xiao2024streamingllm,fu2025headkv,lin2025compresskv,xiao2025duoattention} and retrieval-aware decoding~\citep{gema2025decore}. We do not release a new model, and the method does not enable new generative capabilities.

\textbf{Potential positive impacts.}
A more accurate map of which heads carry non-literal retrieval may enable (i) more aggressive KV cache compression at fixed quality, lowering inference cost and energy use; (ii) more targeted interventions for hallucination mitigation in retrieval-augmented generation~\citep{gema2025decore}; and (iii) cleaner mechanistic accounts of long-context behavior in production-scale models.

\textbf{Potential negative impacts and mitigations.}
The method exposes which heads are causally critical for retrieval, which in principle could be used to construct adversarial inputs that suppress retrieval (a denial-of-capability attack). We judge this risk to be low because: (i) the same information is recoverable from any open-weight model with a few hours of compute, so the marginal uplift from publication is small; (ii) the method requires white-box access to attention weights and value caches, which is not available through standard inference APIs; and (iii) the attack surface (degrading long-context retrieval) is qualitatively similar to risks already present in any KV cache compression method. We therefore do not impose access restrictions on the released code.

\textbf{Fairness and demographic considerations.}
Our evaluation uses only the English NoLiMa benchmark and English parametric controls; we do not evaluate retrieval-head identification across languages, dialects, or demographic axes, and we do not claim that the identified heads generalize to non-English text (\cref{sec:limitations}). Practitioners deploying \methodnameabbrev-derived KV-compression policies on multilingual systems should re-validate on their target languages.

\section{Declaration of LLM Usage}
\label[appendix]{app:llm-usage}

We used LLMs as writing assistants for paper polishing and for routine refactoring of plotting scripts. Specifically: (i) prose passes for clarity and concision, particularly on appendix sections, (ii) LaTeX-formatting suggestions, and (iii) minor refactoring of plotting scripts. The LLM was not used to design experiments, derive the theoretical results, or generate numerical results. All experimental claims and numerical values reported in the paper are produced by the analysis pipeline released with the code.

\newpage

\section{Worked Example}
\label[appendix]{sec:worked-examples}

The following example uses head $(16, 1)$ from Qwen3-8B on a NoLiMa trial where the answer is ``Yuki'' and the needle span $[200, 212)$ (12 tokens) is embedded in a context of $N = 3{,}000$ tokens.
Numerical values below are illustrative, chosen to show the typical scale of $\alpha$, $\varphi$, $\Phi^+$, and $\Phi^-$ for a strongly-retrieving head;
they are not drawn from a specific logged trial.

\textbf{Step 1: Per-position logit contribution (\cref{eq:logit-contribution}).}
At decode step $t{=}1$ producing ``Yuki'' ($y_1$), consider needle position $j = 205$ containing the token ``Kiasma''.
The head's attention weight is $\alpha^{(16,1)}_{1,205} = 0.08$.
After the OV circuit, the per-position output $\mathbf{o}^{(16,1)}_{1,205}$ has a large component aligned with $\mathbf{u}_{\text{Yuki}}$:
\[
    \phi^{(16,1)}_{1,205} = \mathbf{u}_{\text{Yuki}}^\top\, \mathbf{o}^{(16,1)}_{1,205} = 1.3\,.
\]
Despite moderate attention ($\alpha = 0.08$), this position contributes substantially to the ``Yuki'' logit because the OV circuit transforms the ``Kiasma'' representation into an output aligned with the answer direction.
This is exactly the non-literal retrieval that attention-based scoring cannot detect.

\textbf{Step 2: Spatial contrast (\cref{eq:needle-logit}).}
With $N_1 = 3{,}000$ key positions at step $t{=}1$:
\begin{align*}
    \Phi^+ &= \textstyle\sum_{j=200}^{211} \phi^{(16,1)}_{1,j} = 4.7 \quad\text{(total logit contribution from needle)} \\
    \Phi^- &= \tfrac{12}{2988} \textstyle\sum_{j \notin [200,212)} \phi^{(16,1)}_{1,j} = \tfrac{12}{2988} \times 58.3 = 0.23 \quad\text{(length-normalized off-needle)}
\end{align*}
The 12 needle tokens contribute $4.7$ to the ``Yuki'' logit; a comparable 12-token span from the remaining context contributes only $0.23$.

\textbf{Step 3: Aggregation (\cref{eq:global-aggregate}).}
After 200 ROUGE-passing trials contributing a total of 247 answer steps, head $(16, 1)$ has pooled score $S_{16,1} = 3.8$ with 95\% bootstrap CI $[3.1, 4.5]$ and per-trial consistency $0.94$ (positive in 188/200 trials).
A head with $S = 0.2$, CI $[-0.1, 0.5]$, and consistency $0.53$ is not reliably retrieval-active.

\section{Architecture Adaptations}
\label[appendix]{sec:architecture-adaptations}

The logit-contribution equation (\cref{eq:logit-contribution}) is architecture-invariant: any transformer decomposing into attention weights $\alpha$, value vectors $\mathbf{v}$, an output projection $W_O$, and an unembedding matrix $W_U$ provides the four inputs the method requires.
The adaptations needed for different architectures are mechanical, not algorithmic.

\textbf{Grouped-query attention (GQA).}
In GQA models (Qwen3, Gemma-3, OLMo-3.1), each KV group serves $G = H_Q / H_{KV}$ query heads.
The value cache is stored at $H_{KV}$ granularity and expanded to $H_Q$ heads before the per-position projection:
\[
    \tilde{\mathbf{V}}^{(l)}_{h} = \mathbf{V}^{(l)}_{\lfloor h \cdot H_{KV}/H_Q \rfloor}
    \quad \text{for } h = 0, \ldots, H_Q - 1\,.
\]
Query heads within the same group share value vectors but differ in $W_O^{(l,h)}$, so their logit-contribution scores can differ.

\textbf{Vision--language models (Gemma-3).}
The decoder layers and unembedding matrix reside at a nested path (\texttt{language\_model.model.layers} and \texttt{language\_model.lm\_head}).
Per-head QK normalization affects attention distributions but does not change the value vectors or output projections.
Detection operates on text-only self-attention; extending to multimodal inputs is left for future work.

\paragraph{Implementing head ablation in vLLM.}
Our ablation intervention requires a masked pass that zeros the queries of a designated head subset while sharing the same prefix history. vLLM's default generation path is built around paged attention and a continuous-batch scheduler, neither of which exposes the per-head query manipulation the method needs.
We therefore use vLLM specifically for model loading, tensor-parallel weight sharding, and tokenization, and replace its generation loop with a manual autoregressive driver.
At load time, we monkey-patch each supported attention class (\texttt{Gemma3Attention}, \texttt{Qwen3Attention}, \texttt{Olmo2Attention}) so that, when the driver is active, the layer routes through \texttt{F.scaled\_dot\_product\_attention} backed by two independent sequential KV caches---one per pass.
Each decode step runs the forward with retrieval-head queries zeroed.
For tensor-parallel execution ($\mathrm{TP}>1$), the patching and KV bookkeeping run inside every worker, and the two passes are dispatched via vLLM's \texttt{collective\_rpc}; global head indices are remapped to local shard indices per rank.
The trade-off is an $O(n^2)$ KV cache and single-request decoding rather than continuous batching; the benefit is that any model loadable by vLLM under one of the supported attention classes is supported without modifying vLLM internals.

\newpage

\section{Logit-Contribution Score Distribution}
\label[appendix]{app:score-distribution}

\Cref{fig:score-dist} shows the distribution of \methodnameabbrev scores $S_{l,h}$ across all attention heads for each of the six models, with the top-50 and bottom-50 heads highlighted.
In every model, the score distribution is heavily right-skewed: a small number of heads have large positive scores, the vast majority cluster near zero, and the bottom-50 heads all have strictly negative scores.
This confirms that the bottom-$k$ ablation experiments in \cref{sec:bottom-k}---which ablate up to $k{=}50$ heads ranked from most negative---exclusively target heads with $S_{l,h} < 0$, \ie heads whose logit contribution toward the correct answer originates predominantly from off-needle positions.
The clear separation between positive and negative tails across all six models supports the validity of the spatial contrast as a discriminative criterion.

\begin{figure}[t]
  \centering
  \includegraphics[width=0.9\linewidth]{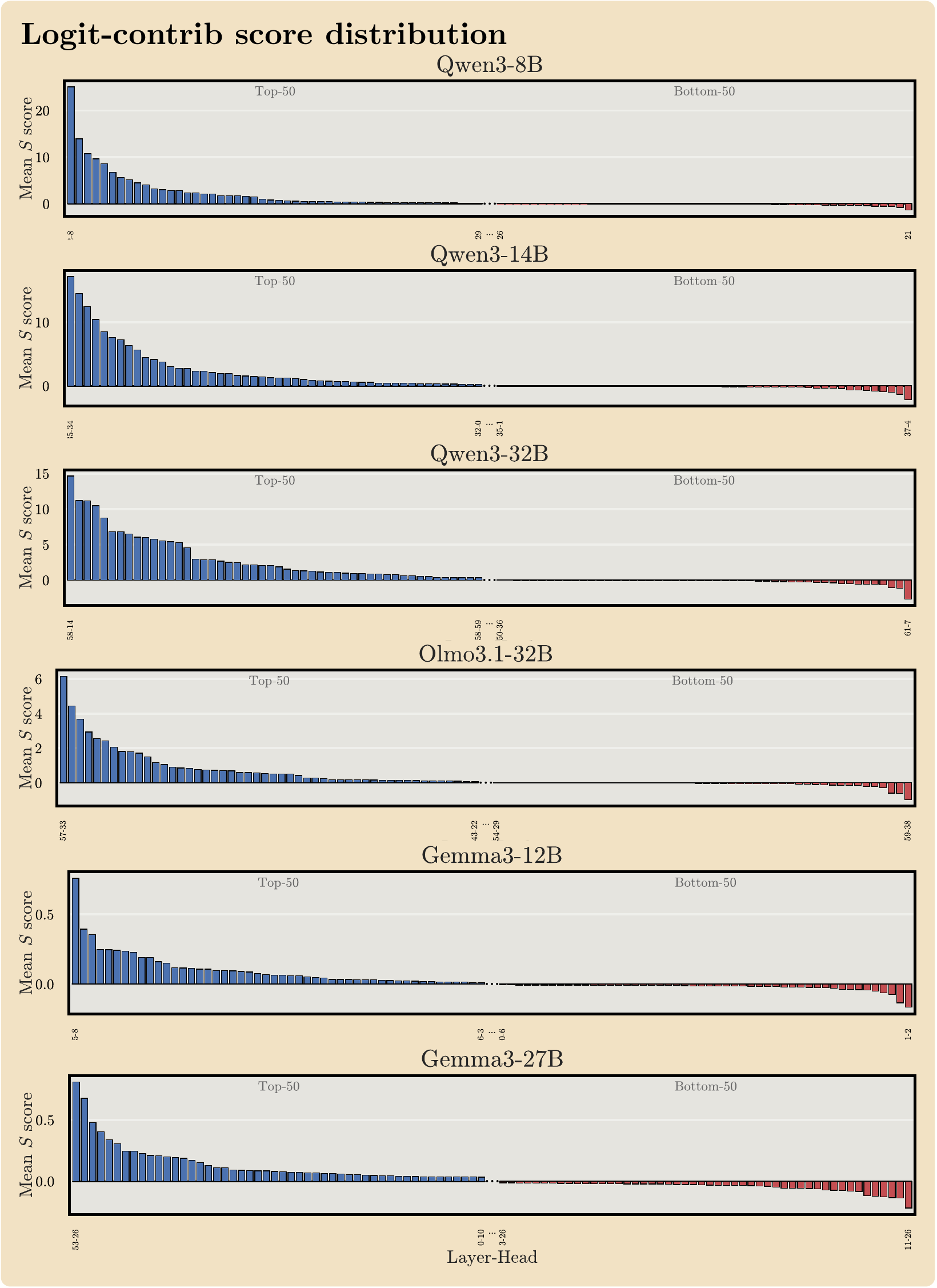}
  \caption{\textbf{Distribution of \methodnameabbrev scores across all heads for each model.}
  Heads are sorted by $S_{l,h}$; the top-50 (blue, left) and bottom-50 (red, right) are highlighted.
  In every model, the bottom-50 heads have strictly negative scores, confirming that the bottom-$k$ experiments (\cref{sec:bottom-k}) exclusively ablate heads whose answer-aligned logit contribution originates from off-needle positions.}
  \label{fig:score-dist}
\end{figure}

\section{Bottom-$k$ Dissociation Score}
\label[appendix]{app:bottomk-ds}

\Cref{fig:bottomk-ds} shows the dissociation score under bottom-$k$ ablation alongside parametric accuracy.
While top-$k$ ablation produces high dissociation scores---retrieval degrades steeply with parametric accuracy intact---bottom-$k$ ablation yields near-zero dissociation at all depths, confirming that neither retrieval nor parametric performance is affected.

\begin{figure}[h]
  \centering
  \includegraphics[width=\linewidth]{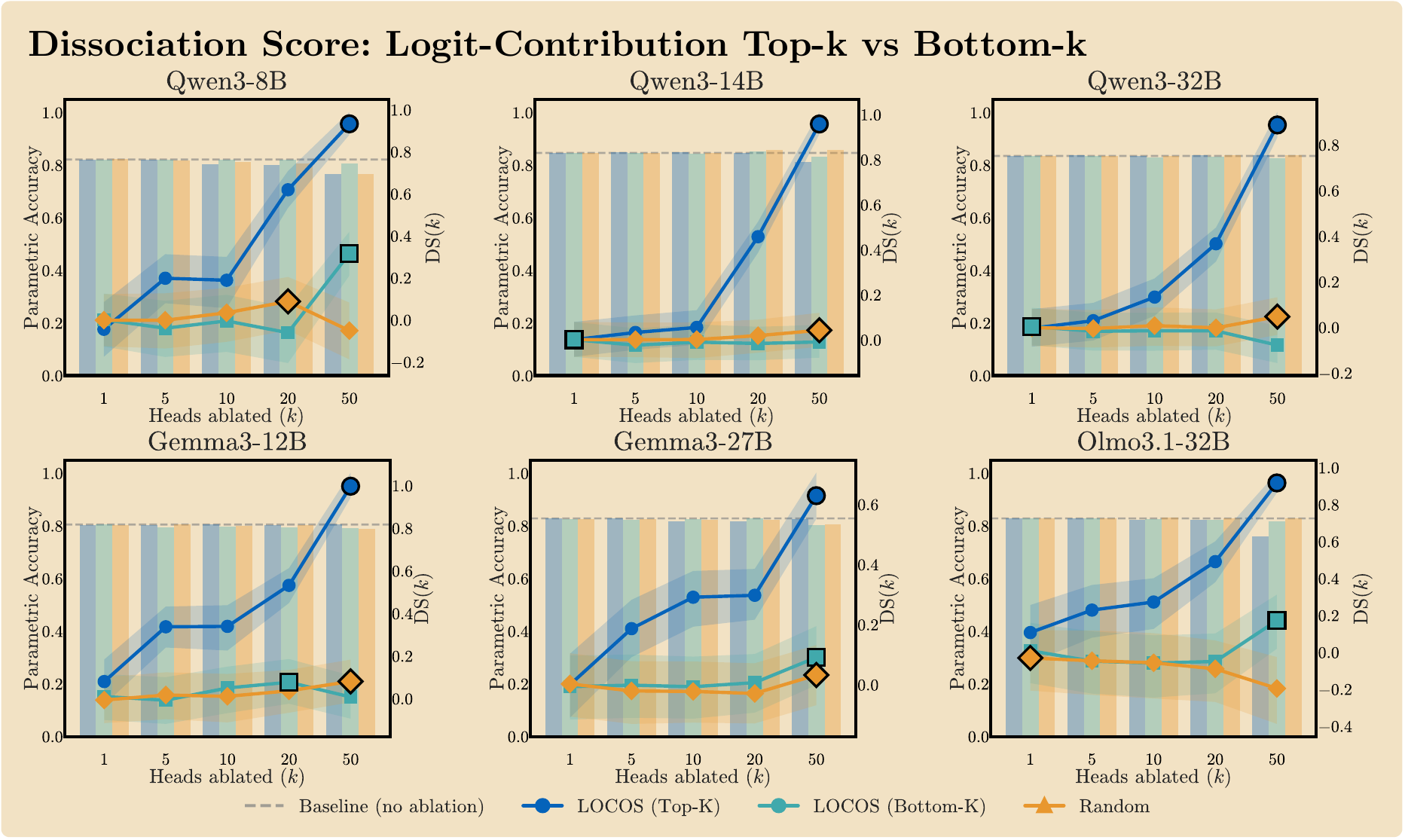}
  \caption{\textbf{Bottom-$k$ ablation produces near-zero dissociation.}
  Dissociation score DS$(k)$ and parametric accuracy as a function of ablation depth~$k$ for bottom-$k$ heads across six models.
  Unlike top-$k$ ablation (\cref{fig:fds}), bottom-$k$ ablation leaves both retrieval and parametric performance near baseline.}
  \label{fig:bottomk-ds}
\end{figure}

\section{KV-Group $\times$ Layer View of Logit-Contribution Scores}
\label[appendix]{app:kv-group-layer}


In GQA models~\citep{ainslie2023gqa}, query heads sharing a KV group consume identical key and value vectors, so their per-head \methodnameabbrev scores are correlated.
The cache unit in GQA is the KV group, not the individual query head, so KV cache footprint is also reasoned about at (layer, KV-group) granularity.
This appendix asks one spatial question that the head-level analyses in \cref{sec:ablation-comparison,sec:bottom-k} do not: at which layers do the highest-scoring KV groups appear?

\begin{figure}[t]
  \centering
  \includegraphics[width=\linewidth]{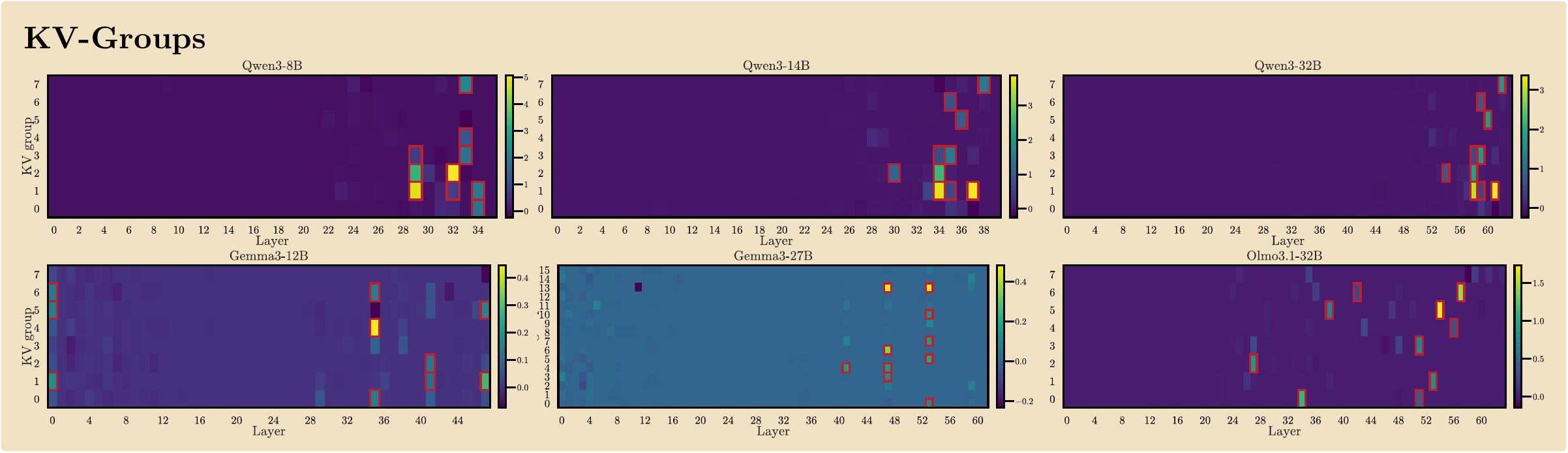}
  \caption{\textbf{Top-10 \methodnameabbrev cells concentrate in late layers in the Qwen3 family on NoLiMa, but span broader layer ranges in Gemma-3-12B and OLMo-3.1-32B.}
  Per-(layer, KV-group) mean \methodnameabbrev score on NoLiMa for Qwen3-8B, Qwen3-14B, Qwen3-32B, OLMo-3.1-32B, Gemma-3-12B, and Gemma-3-27B.
  Layer is on the $x$-axis, KV group on the $y$-axis, color encodes the mean score across passing trials.
  Red boxes mark the top-10 (layer, KV-group) cells per model.
  Per-panel color scales differ; the layer pattern is read from the spatial location of red boxes, not from absolute magnitude.
  Per-model layer spans and KV-group counts are reported in \cref{tab:kv-group-summary}.}
  \label{fig:kv-group-heatmap}
\end{figure}

\textbf{Observations.}
\Cref{tab:kv-group-summary} summarizes the top-10 (layer, KV-group) cells across all six GQA models.
Two patterns emerge.
First, layer concentration is family-dependent.
The Qwen3 family places every top-10 cell in the upper quartile of layers across all three scales (29--34 of 36, 30--38 of 40, and 54--62 of 64), and Gemma-3-27B is concentrated in the upper third (41--53 of 62).
Gemma-3-12B and OLMo-3.1-32B do not show comparable upper-layer concentration: their top-10 cells span 0--47 of 48 and 27--57 of 64 respectively.
Second, KV-group spread is roughly at chance.
The number of distinct KV groups containing top-10 cells (6--8) is at or just above the uniform-allocation expectation ($\approx 5.9$ for 8-group models, $\approx 7.6$ for 16-group models), so the KV-group dimension does not concentrate under \methodnameabbrev in any of the six models.
The evidence is six models from three families on one detector and one benchmark; the family-dependent layer pattern in particular warrants replication on additional benchmarks before generalization.

\begin{table}[h]
\centering
\small
\caption{\textbf{Spatial summary of the top-10 (layer, KV-group) \methodnameabbrev cells on NoLiMa.}
Listed at KV-group granularity, since query heads in a group share keys and values. The rightmost column reports the number of distinct KV groups containing at least one top-10 cell; under uniform allocation of 10 cells the expected count is $\approx 5.9$ for 8-group models and $\approx 7.6$ for 16-group models.}
\label{tab:kv-group-summary}
\resizebox{\linewidth}{!}{%
\begin{tabular}{@{}lcccc@{}}
\toprule
Model & Layers $\times$ KV-groups & Top-10 layer span & Top-10 KV groups & Distinct KV groups in top-10 \\
\midrule
Qwen3-8B      & $36 \times 8$  & layers 29--34 & $\{0, 1, 2, 3, 4, 7\}$         & $6 / 8$  \\
Qwen3-14B     & $40 \times 8$  & layers 30--38 & $\{1, 2, 3, 5, 6, 7\}$         & $6 / 8$  \\
Qwen3-32B     & $64 \times 8$  & layers 54--62 & $\{1, 2, 3, 5, 6, 7\}$         & $6 / 8$  \\
OLMo-3.1-32B  & $64 \times 8$  & layers 27--57 & $\{0, 1, 2, 3, 4, 5, 6\}$      & $7 / 8$  \\
Gemma-3-12B   & $48 \times 8$  & layers 0--47  & $\{0, 1, 2, 4, 5, 6\}$         & $6 / 8$  \\
Gemma-3-27B   & $62 \times 16$ & layers 41--53 & $\{0, 3, 4, 5, 6, 7, 10, 13\}$ & $8 / 16$ \\
\bottomrule
\end{tabular}
}
\end{table}

\textbf{The layer pattern is detector- and metric-conditional.}
The late-layer concentration in \cref{fig:kv-group-heatmap} is not evidence that retrieval circuitry lives only in late layers.
Two alternative explanations remain unaddressed by the figure.
First, the per-position score (\cref{eq:logit-contribution}) is a direct-path projection onto the answer-token unembedding.
Late-layer heads sit closer to the unembedding with fewer downstream non-linearities to redirect their output, so their direct-path projections are systematically larger as a property of the metric, irrespective of whether they implement retrieval.
The tuned-lens variant in \cref{fig:direct-path} addresses this concern at head granularity for Gemma-3-27B; we have not re-run it at KV-group granularity for the two Qwen3 models.
Second, the same KV-group $\times$ layer view computed from Wu (NIAH) attention scores places top cells across a wider layer range, so a ``where retrieval lives'' claim made from \cref{fig:kv-group-heatmap} alone would not survive a different detector.
\cref{fig:kv-group-heatmap} is therefore descriptive of \methodnameabbrev scores at KV-group granularity, not a localization of retrieval circuitry.
Whether these cells implement retrieval rather than some other late-layer behavior is settled by the head-level NoLiMa ablation in \cref{sec:ablation-comparison} and the specificity controls in \cref{sec:specificity}, not by this figure.

\textbf{Implication for KV cache compression.}
The unit of KV cache memory in GQA is the (layer, KV-group) cell, so concentration at this granularity is what saves cache memory; head-level concentration does not, because query heads in a group share keys and values.
A compression policy could keep full-resolution K/V for the small set of top-scoring cells (10 of 288 in Qwen3-8B and 10 of 320 in Qwen3-14B) and apply heavier quantisation, eviction, or window truncation to the rest.
The lower three quarters of layers contain no top-10 cell in either model.
This gives a selection criterion at finer granularity than the per-layer eviction policies of \citet{zhang2023h2o}, \citet{li2024snapkv}, and \citet{xiao2024streamingllm}; whether it produces a better quality--budget trade-off is the experiment we do not run.
\citet{xiao2025duoattention} make the closely related argument that a small set of retrieval heads should retain full KV cache while the rest use streaming attention; the KV-group view sharpens this argument for GQA models, where retention saves cache only at the group level.

\textbf{Caveat for compression.}
The late-layer pattern is detector- and metric-conditional, so a policy derived from \methodnameabbrev alone risks over-compressing mid-layer KV that other detectors flag as relevant.
A practical heuristic should pool signal from multiple retrieval-head detectors rather than rely on a single score.
We do not run a KV-compression experiment in this paper, and promoting the implication above to a claim would require, at minimum, top-$k$ versus random KV-group eviction on long-context perplexity, NIAH, and NoLiMa, with a held-out suite (\eg parametric recall, arithmetic) to check that the policy leaves non-retrieval capabilities intact.

\section{Six-Model Versions of Main-Text Ablation Figures}
\label[appendix]{app:six-model-figures}

The ablation figures in \cref{sec:bottom-k,sec:specificity,sec:literal-vs-nonliteral} show three representative models (one per family). This appendix reports the full six-model versions. The conclusions stated in the main text are based on all six configurations.

\begin{figure}[h]
  \centering
  \includegraphics[width=\linewidth]{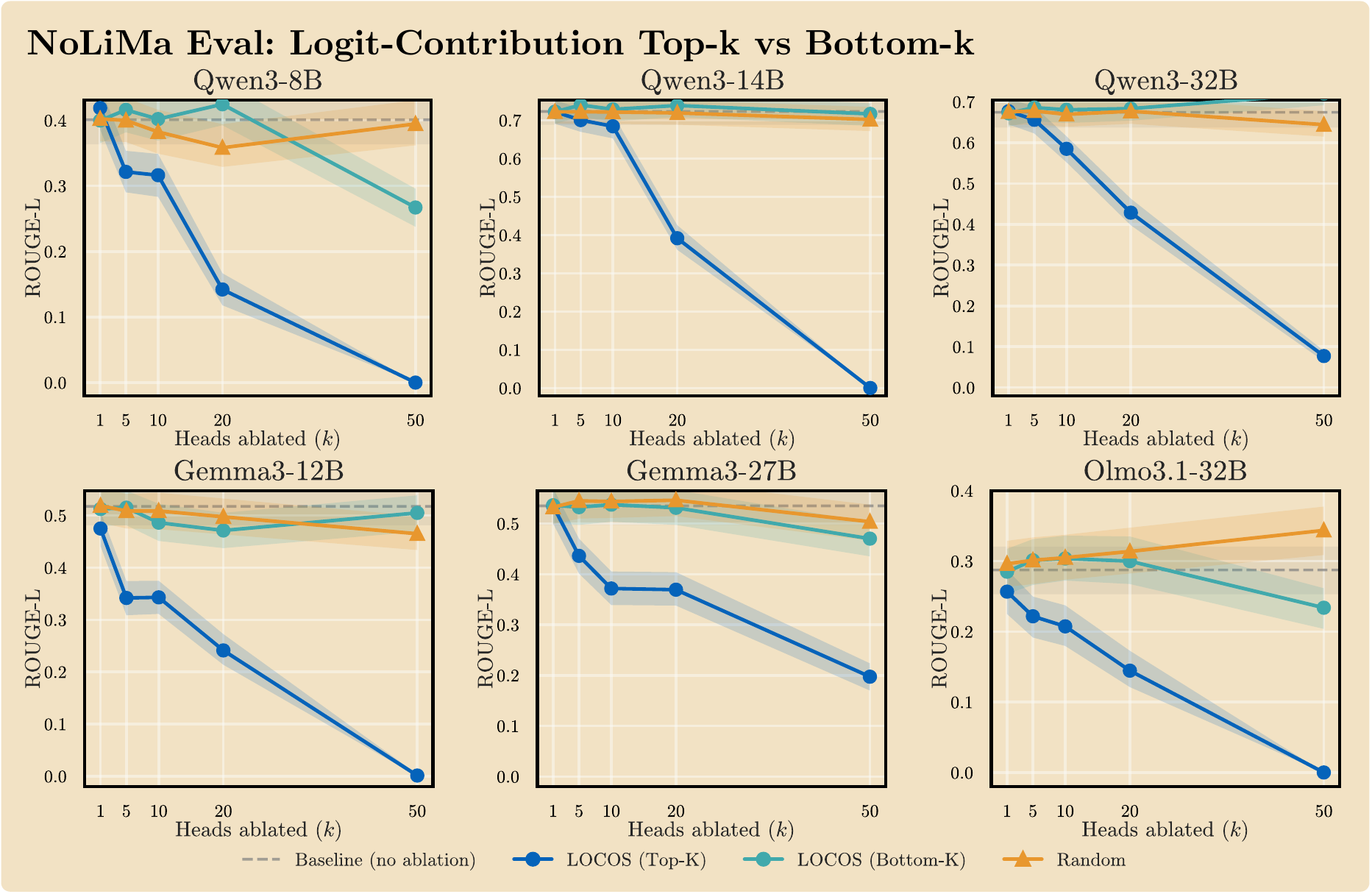}
  \caption{\textbf{Bottom-$k$ ablation does not degrade retrieval (six-model version of \cref{fig:ablation-bottomk}).}
  Each panel shows NoLiMa ROUGE-L as a function of ablation depth~$k$ for top-$k$ (blue), bottom-$k$ (purple), and random heads (red) across all six models.}
  \label{fig:ablation-bottomk-6models}
\end{figure}

\begin{figure}[h]
\centering
\includegraphics[width=\linewidth]{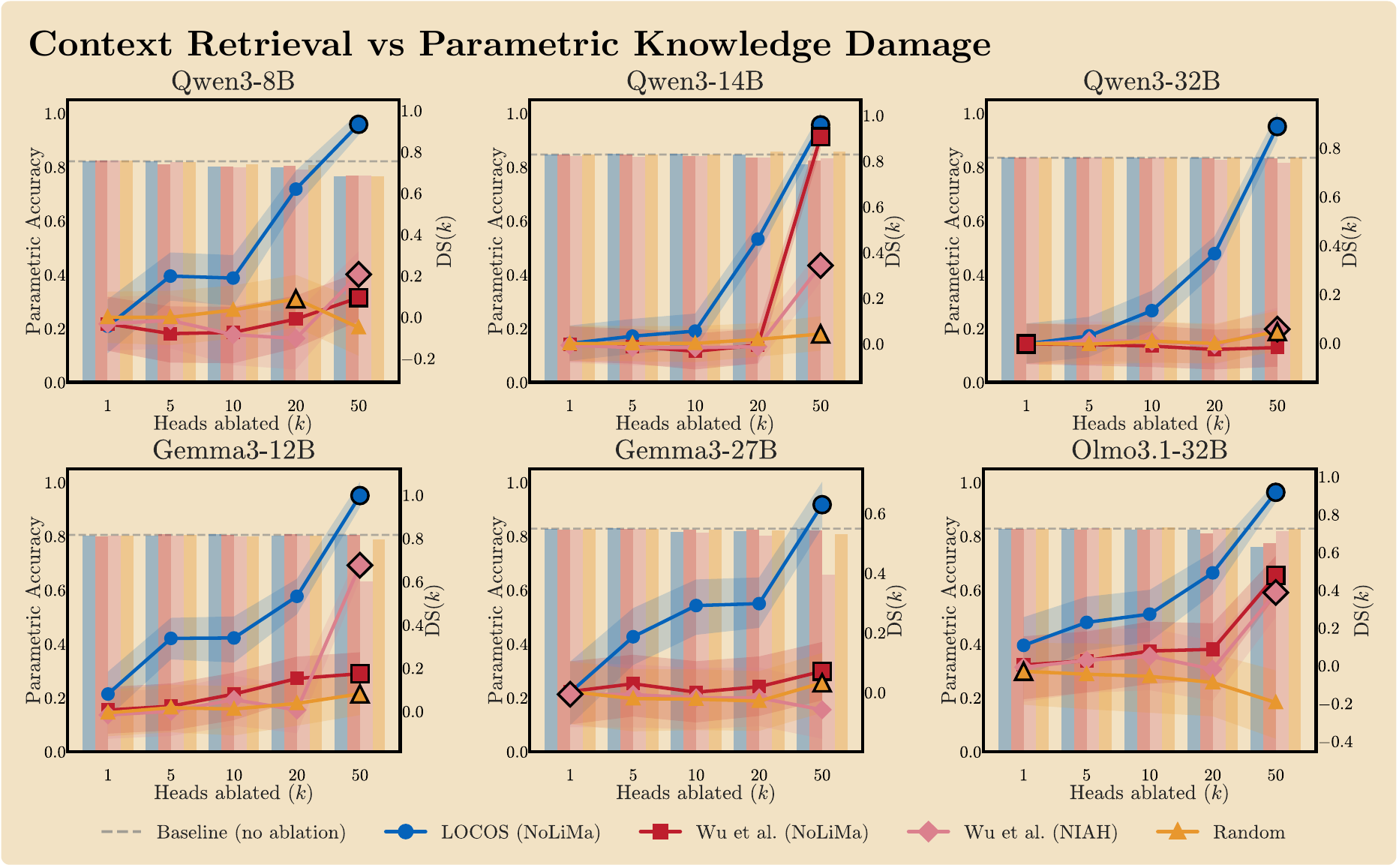}
\caption{\textbf{Functional dissociation between retrieval and parametric capabilities across all six models (six-model version of \cref{fig:fds}).}
Each panel shows DS$(k)$ (lines, right axis) and parametric accuracy (bars, left axis) as a function of ablation depth $k$ for four scoring methods.
\methodnameabbrev (blue) achieves the highest DS in every model configuration; the enlarged marker indicates $k^*$.}
\label{fig:fds-6models}
\end{figure}

\begin{figure}[h]
\centering
\includegraphics[width=\linewidth]{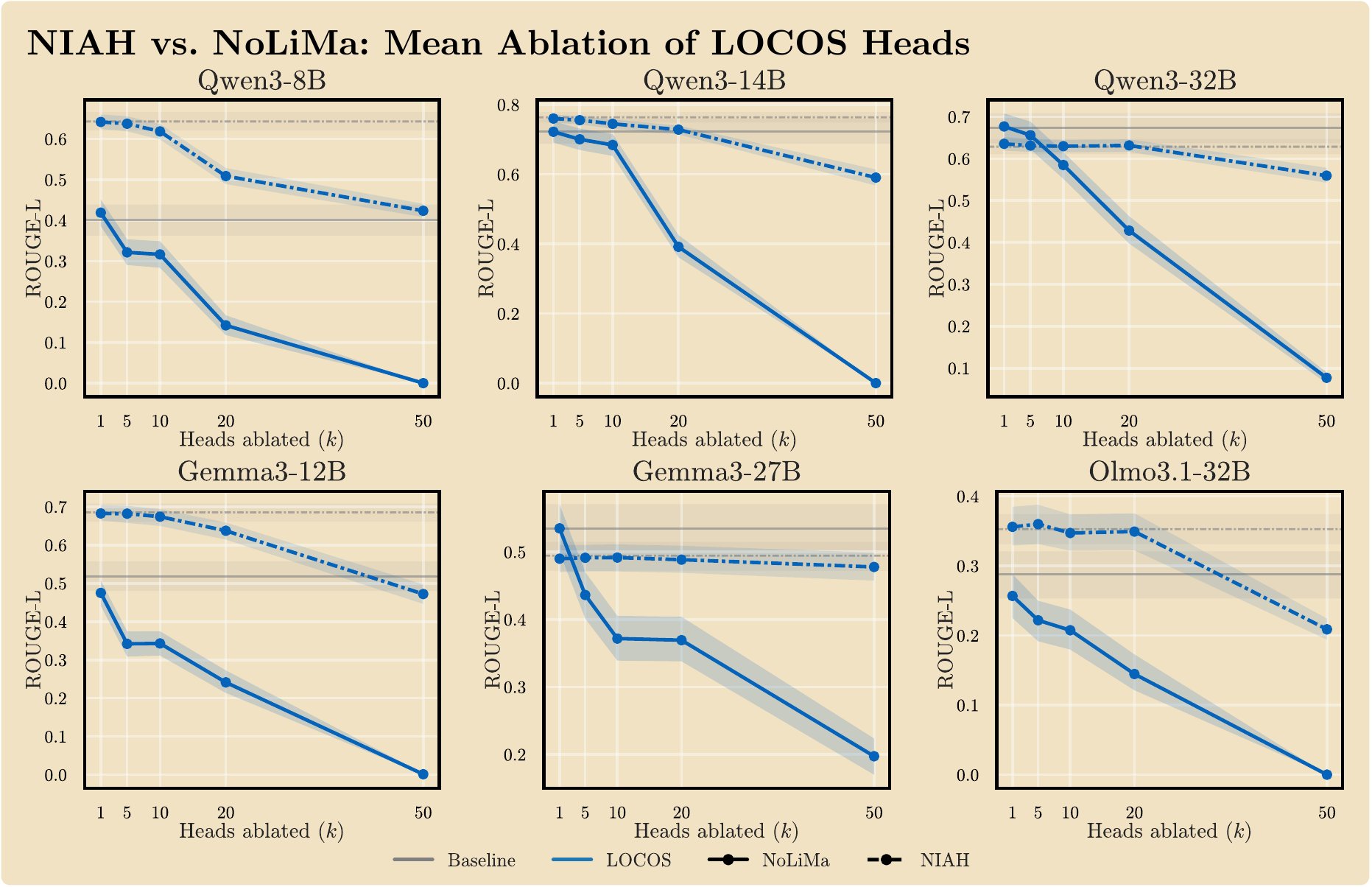}
\caption{\textbf{Non-literal vs.\ literal retrieval damage across all six models (six-model version of \cref{fig:ablation-niah-vs-nolima}).}
Each panel shows ROUGE-L on NoLiMa (solid blue) and standard NIAH (dashed blue) under mean-ablation of the same top-$k$ \methodnameabbrev heads, with the NoLiMa and NIAH baselines marked by solid and dashed gray lines.
The NoLiMa curve declines more steeply than the NIAH curve in every configuration.}
\label{fig:ablation-niah-vs-nolima-6models}
\end{figure}

\section{Downstream Benchmark Example}
\label[appendix]{app:downstream-example}

\prompt{BABILong qa2 (truncated)}{
\textbf{Context:} Mary journeyed to the bathroom. Sandra went to the garden. [\ldots] Daniel journeyed to the bedroom. [\ldots] Daniel took the football there. [\ldots] Daniel dropped the football. [\ldots] Daniel grabbed the football there. [\ldots] Daniel went to the kitchen.\\
\textbf{Question:} Where is the football?\\
\textbf{Answer:} kitchen
}

The example above illustrates why literal token copying is insufficient for BABILong: the correct location (``kitchen'') is the \emph{most recent} of many positions where the entity and object appear, requiring the model to trace a trajectory across interleaved narrative rather than match a unique token.

\newpage

\section{Direct-Path Robustness via Tuned Lens}
\label[appendix]{app:direct-path}

This appendix derives the direct-path identity that motivates the per-position score $\phi^{(l,h)}_{t,j}$ in \cref{eq:logit-contribution}, characterizes what the projection misses, and shows how a tuned-lens variant addresses the bias that affects early-layer heads.

\subsection{The direct-path identity}
\label[appendix]{app:direct-path:identity}

Let $\mathbf{x}^{(l)}_t \in \mathbb{R}^d$ denote the residual stream at layer $l$ and decode step $t$. A standard transformer block updates the residual stream additively,
\begin{equation}\label{eq:residual-update}
\mathbf{x}^{(l)}_t \;=\; \mathbf{x}^{(l-1)}_t \;+\; \sum_{h=1}^{H} \mathbf{a}^{(l,h)}_t \;+\; \mathbf{m}^{(l)}_t,
\end{equation}
where $\mathbf{a}^{(l,h)}_t = \sum_j \mathbf{o}^{(l,h)}_{t,j}$ is the head-$(l,h)$ output (\cref{eq:per-position-output}) and $\mathbf{m}^{(l)}_t$ is the MLP output of layer $l$. We write $\sigma_l(\cdot)$ for the layer-norm rescaling that precedes the unembedding (\eg the final RMSNorm). The next-token logits at step $t$ are
\begin{equation}\label{eq:logit-readout}
\boldsymbol{\ell}_t \;=\; W_U \, \sigma_L\!\bigl(\mathbf{x}^{(L)}_t\bigr) \;\in\; \mathbb{R}^{|\mathcal{V}|}.
\end{equation}

Unrolling \cref{eq:residual-update} from layer $1$ to $L$ and substituting into \cref{eq:logit-readout},
\begin{equation}\label{eq:logit-decomposition}
\boldsymbol{\ell}_t \;=\; W_U\, \sigma_L\!\Bigl( \mathbf{x}^{(0)}_t + \sum_{l=1}^{L}\Bigl[ \sum_{h=1}^{H} \sum_{j=1}^{N_t} \mathbf{o}^{(l,h)}_{t,j} + \mathbf{m}^{(l)}_t \Bigr] \Bigr).
\end{equation}
\Cref{eq:logit-decomposition} is exact. Because $\sigma_L$ is non-linear in general, the contribution of any single $\mathbf{o}^{(l,h)}_{t,j}$ to $\boldsymbol{\ell}_t$ is not additively separable.

\paragraph{The direct-path approximation.} Linearizing $\sigma_L$ at $\mathbf{x}^{(L)}_t$ gives $\sigma_L(\mathbf{x}^{(L)}_t + \mathbf{o}^{(l,h)}_{t,j}) - \sigma_L(\mathbf{x}^{(L)}_t) \approx J_{\sigma_L}(\mathbf{x}^{(L)}_t)\, \mathbf{o}^{(l,h)}_{t,j}$, so the contribution of $\mathbf{o}^{(l,h)}_{t,j}$ to the logit of token $y$ is approximately $\mathbf{u}_y^\top J_{\sigma_L}(\mathbf{x}^{(L)}_t) \mathbf{o}^{(l,h)}_{t,j}$. The \emph{direct-path score} replaces $J_{\sigma_L}$ with the identity, yielding
\begin{equation}\label{eq:direct-path-score}
\phi^{(l,h)}_{t,j} \;\stackrel{\mathrm{def}}{=}\; \mathbf{u}_{y_t}^\top\, \mathbf{o}^{(l,h)}_{t,j}.
\end{equation}
This is the score used throughout the main text (\cref{eq:logit-contribution}). The omitted Jacobian acts close to a per-step rescaling for the RMSNorm models we evaluate (Qwen3, Gemma-3, OLMo-3 are all RMSNorm); we treat the substitution as an approximation, not an identity, and validate it empirically via the tuned-lens variant in \cref{app:direct-path:tunedlens,app:direct-path:empirical}.%

\subsection{What the direct path misses}
\label[appendix]{app:direct-path:misses}

The error in the direct-path score (\cref{eq:direct-path-score}) relative to the exact contribution decomposes into three sources, in increasing order of severity for early-layer heads:

\begin{enumerate}[nosep,leftmargin=*]
    \item \textbf{LayerNorm rescaling.} The final $\sigma_L$ rescales each direction by $\gamma / \rho(\mathbf{x}^{(L)}_t)$. This is approximately constant across heads at a given step but varies across steps; it does not change rankings within a step but inflates magnitudes.
    \item \textbf{Downstream attention re-mixing.} A head $(l, h)$ at layer $l < L$ writes into the residual stream that is read by all heads at layers $l+1, \ldots, L$. If a downstream head $(l', h')$ with $l' > l$ amplifies or cancels the answer-aligned component of $\mathbf{o}^{(l,h)}_{t,j}$, the direct-path score under-counts (or, in the cancellation case, can sign-flip relative to) the true causal contribution.
    \item \textbf{MLP composition.} MLP sublayers between layer $l$ and $L$ apply non-linear transformations (gated activations, GELU/SwiGLU). A retrieval head whose output only becomes answer-aligned after a downstream MLP read--write pair is invisible to any linear probe of $\mathbf{x}^{(l)}_t$.
\end{enumerate}

Sources (1) and (2) are linear in $\mathbf{o}^{(l,h)}_{t,j}$ and are therefore correctable in principle by a learned linear lens. Source (3) is non-linear and is fundamentally outside the reach of any linear probe.

\subsection{Tuned-lens formalism}
\label[appendix]{app:direct-path:tunedlens}

The tuned lens of \citet{belrose2025tunedlens} learns, for each layer $l$, an affine map $T_l : \mathbb{R}^d \to \mathbb{R}^d$ such that $W_U\, T_l(\mathbf{x}^{(l)}_t)$ approximates the model's true next-token logits $\boldsymbol{\ell}_t$. Concretely, $T_l(\mathbf{x}) = A_l \mathbf{x} + b_l$, with $(A_l, b_l)$ trained by minimizing the KL divergence between $\mathrm{softmax}(W_U T_l(\mathbf{x}^{(l)}_t))$ and the true output distribution on a held-out corpus. The tuned lens absorbs sources (1) and (2) of \cref{app:direct-path:misses} into the learned linear map; source (3) remains uncaptured.

\paragraph{Lens-corrected score.} Substituting the tuned lens for the identity in \cref{eq:direct-path-score} yields a layer-aware score
\begin{equation}\label{eq:tuned-lens-score}
\phi^{(l,h),\mathrm{TL}}_{t,j} \;\stackrel{\mathrm{def}}{=}\; \mathbf{u}_{y_t}^\top\, A_l\, \mathbf{o}^{(l,h)}_{t,j} \;+\; \text{constant in } j,
\end{equation}
where the bias $b_l$ contributes a $j$-independent term that cancels under the spatial contrast (\cref{eq:needle-logit}). Writing $\tilde{\mathbf{p}}^{(l,h)} = (W_O^{(l,h)})^\top A_l^\top \mathbf{u}_{y_t}$, the lens-corrected per-position score takes the same factored form as the direct-path score:
\begin{equation*}
\phi^{(l,h),\mathrm{TL}}_{t,j} \;=\; \alpha^{(l,h)}_{t,j} \cdot \bigl(\mathbf{v}^{(l,h)}_{t,j}\bigr)^\top \tilde{\mathbf{p}}^{(l,h)} \;+\; \text{const.},
\end{equation*}
so the same precomputation-and-batched-inner-product implementation applies; only the projection vector $\tilde{\mathbf{p}}$ changes.

\paragraph{What the lens does not fix.} \Cref{eq:tuned-lens-score} replaces the direct-path Jacobian by a layer-aware linear approximation, but it cannot capture any non-linear composition (source (3)). A head whose write only becomes answer-aligned after passing through a downstream MLP scores low under both the direct-path and tuned-lens variants. Distinguishing a genuinely causally inert head from one whose contribution is hidden by downstream non-linearity requires a non-linear probe, of which the canonical example is causal activation patching~\citep{wang2023interpretability,meng2022locating}.

\subsection{Empirical comparison}
\label[appendix]{app:direct-path:empirical}

\begin{wrapfigure}{r}{0.5\textwidth}
  \centering
  \vspace{-3.5em}
  \includegraphics[width=0.5\textwidth]{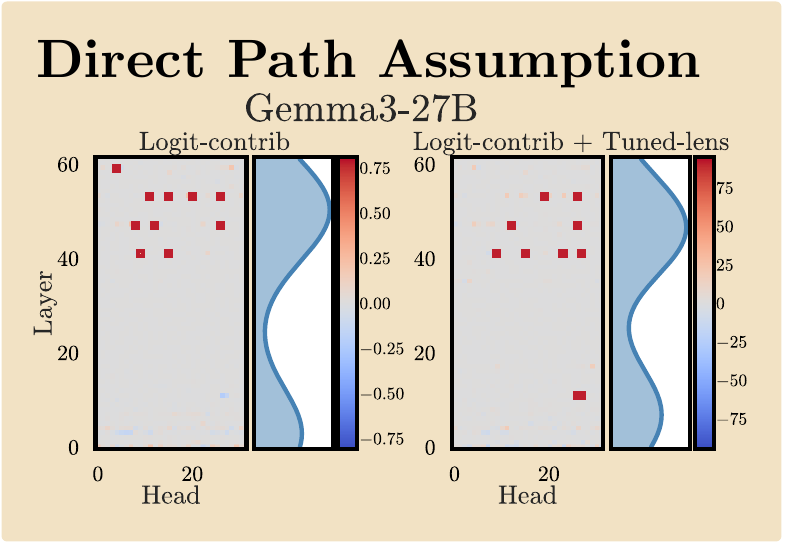}
  \caption{\textbf{Late-layer concentration persists under tuned-lens projection.}
  Heatmaps for Gemma-3-27B: direct-path \methodnameabbrev (left) vs.\ tuned-lens variant (right). Both methods concentrate high-scoring heads in layers 35--60; the layer-marginal distributions peak in the same band. The tuned-lens variant surfaces two additional heads at layer~11 (heads~26 and~27) that do not appear in the direct-path top-$k$ set, but does not produce a broader redistribution toward earlier layers. The score magnitudes differ ($\pm 75$ vs.\ $\pm 0.75$) because the learned affine map amplifies contributions by accounting for downstream transformations.}
  \label{fig:direct-path}
  \vspace{-1.5em}
\end{wrapfigure}

\Cref{fig:direct-path} compares layer$\times$head score heatmaps for Gemma-3-27B under \cref{eq:direct-path-score} and \cref{eq:tuned-lens-score}. The agreement on the late-layer band ($l \in [35, 60]$) suggests that the late-layer concentration reflects retrieval structure rather than a systematic underestimation of early-layer heads by the direct-path projection. Two layer-11 heads (26 and 27) surface only under the tuned lens; we treat these as candidate early-layer retrieval heads that the direct-path score under-ranks but the linear correction recovers.

\clearpage

\subsection{Gemma-3-27B inversion}
\label[appendix]{app:direct-path:gemma-inversion}

\begin{wrapfigure}{r}{0.5\textwidth}
  \centering
  \vspace{-1.5em}
  \includegraphics[width=\linewidth]{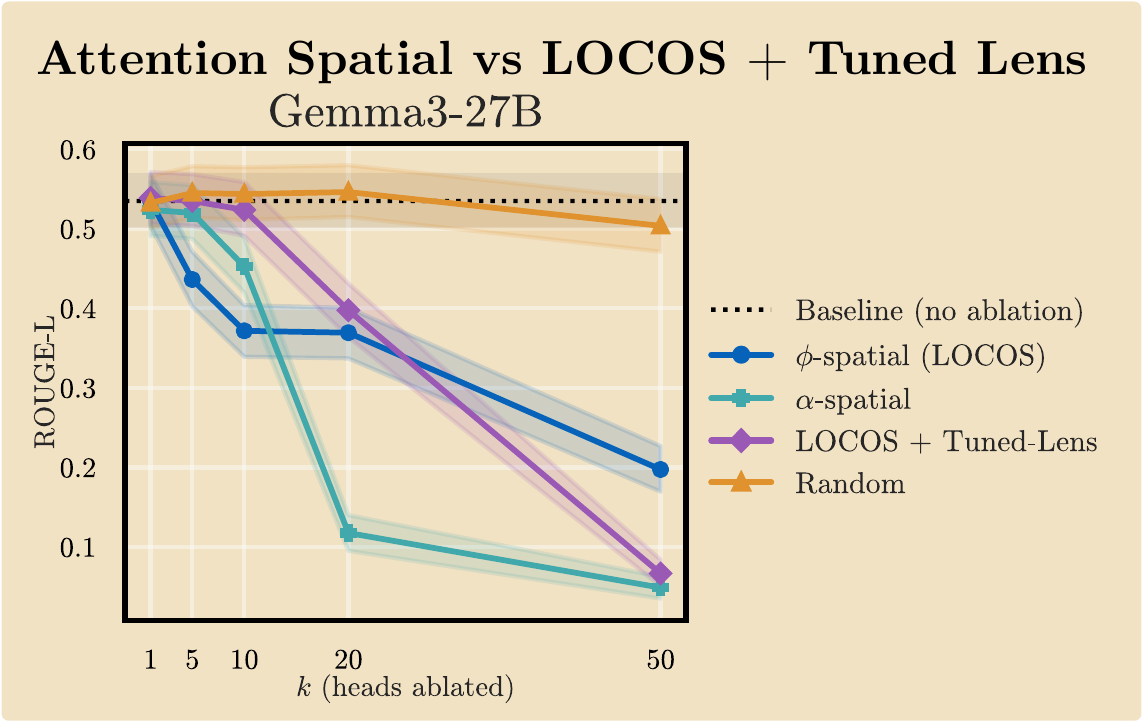}
  \caption{\textbf{Tuned-lens correction only partly resolves the Gemma-3-27B inversion.}
  NoLiMa ROUGE-L under mean-ablation of top-$k$ heads ranked by direct \methodnameabbrev, the attention-only spatial-contrast control, and the tuned-lens-corrected \methodnameabbrev variant on Gemma-3-27B.
  The tuned-lens variant closes much of the gap with attention-only scoring at large $k$, but direct \methodnameabbrev selects the most damaging heads at small $k$.}
  \label{fig:gemma-inversion-tuned-lens}
  \vspace{-1.0em}
\end{wrapfigure}

Replacing the direct-path projection with the tuned-lens readout largely closes the gap with the attention-only control at large $k$ on Gemma-3-27B (\cref{fig:gemma-inversion-tuned-lens}).
At small $k$, however, direct \methodnameabbrev selects more damaging heads than either $\alpha$-spatial scoring or the tuned-lens variant.
Thus, the Gemma-3-27B anomaly is not cleanly explained as a direct-path readout artifact: correcting the readout helps only at the deepest sweep and worsens the top-ranked heads.
This suggests that, for this model, spatial attention contrast captures a causal signal that the evaluated linear write-projections do not rank consistently (\cref{app:direct-path:misses}).
For Gemma-3-27B, attention placement under spatial contrast carries more causal signal than any linear write-projection we evaluate.

\subsection{Beyond linear probes}
\label[appendix]{app:direct-path:beyond}

A linear probe (direct path or tuned lens) cannot detect a head whose contribution to $y_t$ materializes only through MLP composition. Per-head causal activation patching is the canonical gradient-free, non-linear-aware test: for each head $(l, h)$, run a clean forward pass on the needle prompt and a corrupt forward pass with the needle removed or scrambled, then patch only the $(l, h)$ activation from the clean run into the corrupt run and measure the change in logit difference at the answer position,
\begin{equation}\label{eq:causal-attribution}
    \mathrm{CA}(l, h) \;=\; \bigl[\ell^{\mathrm{patch}}_{y^{*}} - \ell^{\mathrm{patch}}_{y^{\mathrm{cf}}}\bigr] - \bigl[\ell^{\mathrm{corr}}_{y^{*}} - \ell^{\mathrm{corr}}_{y^{\mathrm{cf}}}\bigr],
\end{equation}
where $y^{*}$ is the gold first answer token and $y^{\mathrm{cf}} = \arg\max_{y \neq y^{*}} \ell^{\mathrm{corr}}_{y}$ is the strongest non-gold competitor under the corrupt baseline. We refer to this score as \emph{causal attribution} (CA).

\begin{figure}[h]
  \centering
  \includegraphics[width=\linewidth]{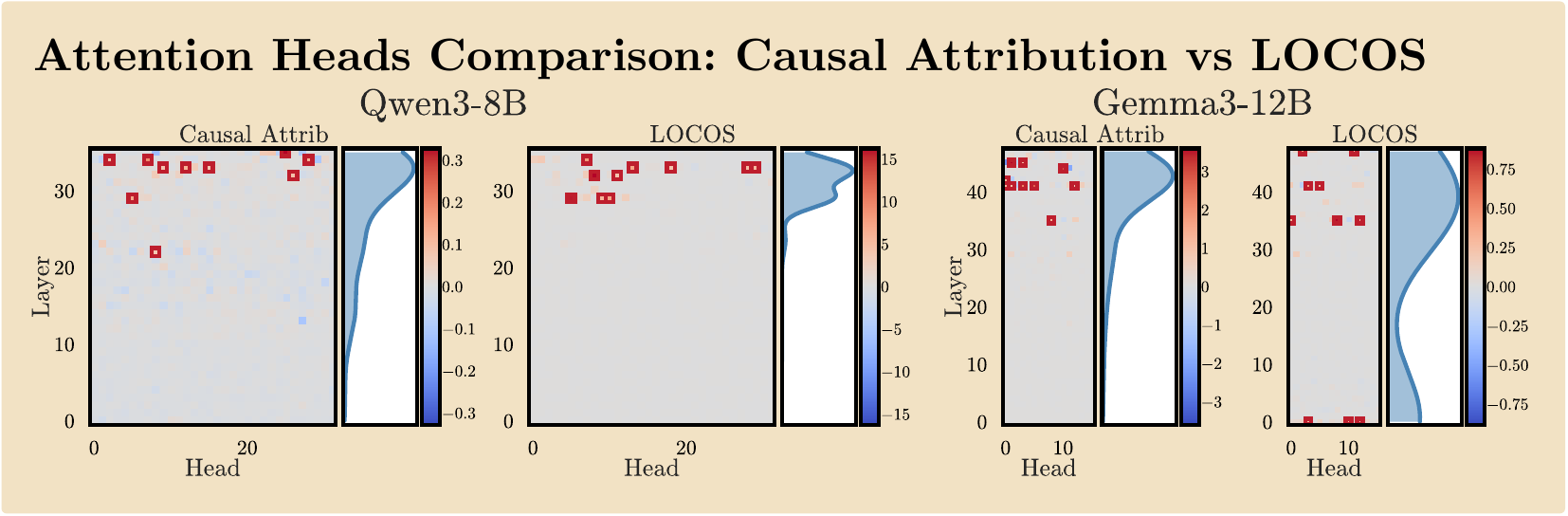}
  \caption{\textbf{Causal attribution vs.\ \methodnameabbrev top-10 heads on Qwen3-8B and Gemma-3-12B.}
  Per-(layer, head) score heatmaps with red boxes marking each method's top-10 cells; layer-marginal kernel densities on the right of each panel.
  Both methods concentrate top-10 heads in the upper layers in both models, but the top-10 sets overlap only marginally (2/10 for Qwen3-8B, 3/10 for Gemma-3-12B).
  On Gemma-3-12B, \methodnameabbrev surfaces several layer-0 heads that causal attribution does not flag.}
  \label{fig:ca-vs-lc}
\end{figure}

\textbf{Empirical comparison on two models.}
\Cref{fig:ca-vs-lc} reports causal attribution and \methodnameabbrev side by side for Qwen3-8B and Gemma-3-12B.
The shared finding is that the late-layer concentration is not a direct-path artifact: causal attribution---which is non-linear-aware and gradient-free---also places its top-10 heads predominantly in the upper layers in both models (layers 22--35 for Qwen3-8B, concentrated in 33--35; layers 35--45 for Gemma-3-12B, concentrated near 41).
\methodnameabbrev's top-10 sets sit in similar bands (29--34 for Qwen3-8B; layers spanning 0, 35, 41, 47 for Gemma-3-12B).
Top-10 overlap is small (2/10 and 3/10), which is consistent with the two scores capturing different aspects of retrieval-related computation rather than one being a strict refinement of the other.

\textbf{Where the methods diverge, and why.}
The notable disagreement on Gemma-3-12B is that \methodnameabbrev surfaces heads at layer~0 that causal attribution does not.
Two explanations are consistent with the figure and we cannot distinguish them here.
First, layer-0 heads may be a direct-path artifact: their write enters the residual stream before any downstream MLP can either amplify or cancel the answer-aligned component, so the linear projection onto $\mathbf{u}_{y_t}$ records a magnitude that the model itself does not realize at the output (sources (2)--(3) of \cref{app:direct-path:misses}).
Second, layer-0 heads may participate in distributed circuits where ablating any single head leaves performance intact (because parallel paths compensate), so causal attribution---a single-head intervention---under-counts them while \methodnameabbrev correctly registers each head's contribution to the linear readout.
The bottom-$k$ control (\cref{sec:bottom-k}) and the top-$k$ ablation curves (\cref{fig:fds}) bound the issue at the level of method validation: \methodnameabbrev-selected top heads, including any artifactual ones, are causally critical \emph{collectively} under group ablation even if individual heads are not under single-head patching.

\textbf{Scope.}
The comparison is two models from two families on one benchmark with one alternative detector.
Causal attribution is also expensive: it requires one additional forward pass per head per trial (a clean--corrupt patching pair scaled by the number of heads, $H_Q \cdot L$ per trial), so the cost grows linearly with the head count and is roughly two orders of magnitude above a single \methodnameabbrev pass for the models we evaluate.
We therefore do not run it on the other four models or on NIAH/parametric controls; running it across the full evaluation matrix is left to future work.
The takeaway we draw is narrow: the late-layer concentration of \methodnameabbrev top heads is corroborated by a non-linear-aware detector on the two models for which we have both scores, and the residual disagreement (notably the Gemma-3-12B layer-0 heads) is a candidate site for future single-head causal validation rather than a refutation of the present paper's group-ablation results.

\section{Relationship to Attention-Based Scoring}
\label[appendix]{app:relationship-attention}

This appendix gives a self-contained derivation of how \methodnameabbrev relates to the attention-pattern observable used by prior detection methods~\citep{wu2025retrieval,fu2025headkv,lin2025compresskv}. \Cref{app:rel:setup} fixes notation; \cref{app:rel:decomposition} decomposes the per-position OV output into answer-aligned and answer-orthogonal components; \cref{app:rel:reduction} states and proves the reduction of \methodnameabbrev to attention-based scoring under a position-independent OV output; \cref{app:rel:wu,app:rel:headkv} cast the Wu and HeadKV/CompressKV scores as special cases under additional assumptions; \cref{app:rel:spatial-vs-temporal} contrasts spatial and temporal aggregation; \cref{app:rel:gradient} gives the gradient interpretation; and \cref{app:rel:summary} consolidates the results.

\subsection{Setup}
\label[appendix]{app:rel:setup}

We work with the standard pre-softmax decomposition of a transformer attention head~\citep{elhage2021mathematical}. At decode step $t$, head $(l, h)$ produces query $\mathbf{q}^{(l,h)}_t$ and, for each source position $j \in \{1, \ldots, N_t\}$, key $\mathbf{k}^{(l,h)}_{t,j}$ and value $\mathbf{v}^{(l,h)}_{t,j} \in \mathbb{R}^{d_h}$. The QK circuit produces
\begin{equation}\label{eq:qk-softmax}
\alpha^{(l,h)}_{t,j} \;=\; \frac{\exp\bigl(\mathbf{q}^{(l,h){\top}}_t \mathbf{k}^{(l,h)}_{t,j} / \sqrt{d_h}\bigr)}{\sum_{j'=1}^{N_t} \exp\bigl(\mathbf{q}^{(l,h){\top}}_t \mathbf{k}^{(l,h)}_{t,j'} / \sqrt{d_h}\bigr)},
\end{equation}
and the OV circuit writes $\mathbf{o}^{(l,h)}_{t,j} = \alpha^{(l,h)}_{t,j} \cdot W_O^{(l,h)} \mathbf{v}^{(l,h)}_{t,j}$ to the residual stream (\cref{eq:per-position-output}). The full head output is $\mathbf{a}^{(l,h)}_t = \sum_{j=1}^{N_t} \mathbf{o}^{(l,h)}_{t,j}$.

For brevity, we drop the $(l, h)$ superscripts in this appendix when no ambiguity arises and write $\alpha_j$, $\mathbf{v}_j$, $\mathbf{w}_j \stackrel{\mathrm{def}}{=} W_O \mathbf{v}_j$, and $\mathbf{u} \stackrel{\mathrm{def}}{=} \mathbf{u}_{y_t}$. With this notation the per-position score \cref{eq:logit-contribution} is $\phi_j = \alpha_j \cdot \mathbf{u}^\top \mathbf{w}_j$.

\subsection{Parallel--orthogonal decomposition of the OV output}
\label[appendix]{app:rel:decomposition}

Decompose each unweighted OV output $\mathbf{w}_j \in \mathbb{R}^d$ along $\mathbf{u}$ and orthogonal to it:
\begin{equation}\label{eq:parallel-perp}
\mathbf{w}_j \;=\; c_j \cdot \mathbf{u} \;+\; \mathbf{w}^{\perp}_j,
\qquad
c_j \;=\; \frac{\mathbf{u}^\top \mathbf{w}_j}{\|\mathbf{u}\|^2},
\qquad
\mathbf{u}^\top \mathbf{w}^{\perp}_j = 0.
\end{equation}
The scalar $c_j$ is the answer-aligned write magnitude at source position $j$; the residual $\mathbf{w}^{\perp}_j$ writes into directions orthogonal to the answer token. Substituting \cref{eq:parallel-perp} into the per-position score gives the identity
\begin{equation}\label{eq:phi-cj}
\phi_j \;=\; \alpha_j \cdot \|\mathbf{u}\|^2 \cdot c_j,
\end{equation}
which factorises cleanly into \emph{where} the head reads ($\alpha_j$) and \emph{what answer-aligned content} it writes from that position ($c_j$). The factor $\|\mathbf{u}\|^2$ is constant across heads and positions and cancels under any cross-head ranking, so we may treat $\phi_j \propto \alpha_j c_j$ for the purposes of head selection.

\begin{lemma}[Sufficient statistic]
\label{lemma:sufficient}
For any score that aggregates per-position OV writes against the answer direction $\mathbf{u}$, the pair $(\alpha_j, c_j)$ is sufficient: any answer-orthogonal component $\mathbf{w}^\perp_j$ contributes zero.
\end{lemma}
\begin{proof}
Immediate from $\mathbf{u}^\top \mathbf{w}^{\perp}_j = 0$ and the linearity of $\mathbf{u}^\top \mathbf{o}_j = \alpha_j \mathbf{u}^\top \mathbf{w}_j$.
\end{proof}

\Cref{lemma:sufficient} shows that any modification of $\phi$ that is linear in the residual stream and projects onto $\mathbf{u}$---for instance the tuned-lens variant of \cref{app:direct-path}---inherits the same factorisation, with $c_j$ replaced by a lens-corrected scalar.

\subsection{Reduction to attention-based scoring}
\label[appendix]{app:rel:reduction}

\begin{proposition}[Reduction under position-independent OV write]
\label{prop:reduction}
Suppose head $(l, h)$ writes a position-independent answer-aligned magnitude at step $t$, i.e., there exists a scalar $c_t$ such that
\begin{equation}\label{eq:position-independent}
c_j \;=\; c_t \quad \text{for every } j \in \{1, \ldots, N_t\}.
\end{equation}
Then the spatial-contrast score (\cref{eq:needle-logit}) reduces to
\begin{equation}\label{eq:reduction}
\Phi^{+}_t - \Phi^{-}_t \;=\; \|\mathbf{u}\|^2 \cdot c_t \cdot \Bigl(M^{+}_t \;-\; \tfrac{e_\tau - s_\tau}{N_t - (e_\tau - s_\tau)} \cdot M^{-}_t \Bigr),
\end{equation}
where $M^{+}_t = \sum_{j \in [s_\tau, e_\tau)} \alpha_j$ and $M^{-}_t = \sum_{j \notin [s_\tau, e_\tau)} \alpha_j$ are the needle and off-needle attention masses respectively. Up to the per-step scalar $\|\mathbf{u}\|^2 c_t$, \methodnameabbrev coincides with the length-normalized attention-mass contrast between needle and off-needle positions.
\end{proposition}
\begin{proof}
Under \cref{eq:position-independent}, \cref{eq:phi-cj} reads $\phi_j = \|\mathbf{u}\|^2 c_t \cdot \alpha_j$ with the scalar $\|\mathbf{u}\|^2 c_t$ independent of $j$. Substituting into the definitions of $\Phi^{\pm}_t$ from \cref{eq:needle-logit},
\begin{align*}
\Phi^{+}_t &= \|\mathbf{u}\|^2 c_t \sum_{j \in [s_\tau, e_\tau)} \alpha_j \;=\; \|\mathbf{u}\|^2 c_t \, M^{+}_t,\\
\Phi^{-}_t &= \tfrac{e_\tau - s_\tau}{N_t - (e_\tau - s_\tau)} \|\mathbf{u}\|^2 c_t \sum_{j \notin [s_\tau, e_\tau)} \alpha_j \;=\; \tfrac{e_\tau - s_\tau}{N_t - (e_\tau - s_\tau)} \|\mathbf{u}\|^2 c_t \, M^{-}_t.
\end{align*}
Subtracting yields \cref{eq:reduction}.
\end{proof}

\begin{corollary}[Sign of the reduction]
\label{cor:sign-reduction}
Under \cref{prop:reduction}, $\mathrm{sign}(\Phi^{+}_t - \Phi^{-}_t) = \mathrm{sign}(c_t) \cdot \mathrm{sign}\bigl(M^{+}_t - \tfrac{e_\tau - s_\tau}{N_t - (e_\tau - s_\tau)} M^{-}_t\bigr)$. A head that reads from the needle ($M^{+}_t$ exceeds the length-normalized off-needle mass) and writes a positive answer-aligned direction ($c_t > 0$) is correctly assigned a positive \methodnameabbrev score; a head that suppresses the answer logit ($c_t < 0$) receives a negative score.
\end{corollary}

\paragraph{Where the reduction holds.} \Cref{eq:position-independent} requires the head to write the same answer-aligned magnitude $c_t$ at every source position. This is the defining property of a literal-copy head: an induction head $[A][B]\ldots[A] \mapsto [B]$ writes the ``advance the previous token by one step'' direction independently of which $[A]$ position it attended to. In our notation, $\mathbf{w}_j$ depends on $j$ only through $\mathbf{v}_j$, and for a literal-copy head $\mathbf{u}^\top W_O \mathbf{v}_j$ takes the same value whenever $\mathbf{v}_j$ encodes the answer token.

\paragraph{Where the reduction fails.} For a non-literal retrieval head, $c_j$ depends on whether $j$ lies inside the needle (where the value vector encodes the semantic concept that the head must transform into the answer direction) or outside (where the value encodes unrelated context). \Cref{eq:position-independent} breaks, and \cref{prop:reduction} no longer holds: \methodnameabbrev sees this position dependence through the per-position $c_j$ in \cref{eq:phi-cj}, while attention-based scoring observes only $\alpha_j$ and is invariant to $c_j$.

\subsection{Wu's token-matching score as a special case}
\label[appendix]{app:rel:wu}

The score of \citet{wu2025retrieval} assigns credit at decode step $t$ to head $(l, h)$ on a literal-NIAH trial when (i) the head's argmax attention position falls within the needle span and (ii) the token at that position matches the generated token. Formally, with $j^{*} = \arg\max_j \alpha^{(l,h)}_{t,j}$ and $x_j$ the input token at position $j$,
\begin{equation}\label{eq:wu-score}
\mathrm{Wu}^{(l,h)}_t \;=\; \mathbf{1}\!\bigl[\, j^{*} \in [s_\tau, e_\tau) \,\bigr] \cdot \mathbf{1}\!\bigl[\, x_{j^{*}} = y_t \,\bigr].
\end{equation}

\begin{proposition}[Wu's score as a hard-thresholded special case]
\label{prop:wu}
Suppose three conditions hold for head $(l, h)$ at step $t$ on a literal-NIAH trial:
\begin{itemize}[nosep]
    \item[\textup{(W1)}] The head is a literal-copy head: $W_O \mathbf{v}_j \approx \mathbf{u}_{x_j}$, the unembedding row of the attended token, so $\mathbf{u}_{y_t}^\top W_O \mathbf{v}_j \propto \mathbf{u}_{y_t}^\top \mathbf{u}_{x_j}$.
    \item[\textup{(W2)}] The unembedding rows are approximately orthogonal across distinct tokens: $\mathbf{u}_{y_t}^\top \mathbf{u}_{x_j} \approx \|\mathbf{u}_{y_t}\|^2 \cdot \mathbf{1}[x_j = y_t]$.
    \item[\textup{(W3)}] Attention is concentrated: $\alpha_{j^{*}} \gg \alpha_j$ for $j \neq j^{*}$, so the spatial-contrast score is dominated by the argmax position.
\end{itemize}
Then $\mathrm{Wu}^{(l,h)}_t = 1 \iff \Phi^{+}_t - \Phi^{-}_t > 0$ to leading order, and the rankings induced by Wu's score and \methodnameabbrev coincide on the trial.
\end{proposition}
\begin{proof}
Under (W1) and (W2), $c_j \approx \mathbf{1}[x_j = y_t]$. The needle contribution becomes $\Phi^{+}_t \approx \|\mathbf{u}\|^2 \sum_{j \in [s_\tau, e_\tau),\, x_j = y_t} \alpha_j$ and the off-needle contribution $\Phi^{-}_t \approx \tfrac{e_\tau - s_\tau}{N_t - (e_\tau - s_\tau)} \|\mathbf{u}\|^2 \sum_{j \notin [s_\tau, e_\tau),\, x_j = y_t} \alpha_j$. On a literal-NIAH trial the answer token $y_t$ appears predominantly at one position inside the needle, so $\Phi^{-}_t$ is negligible. Under (W3), $\Phi^{+}_t \approx \|\mathbf{u}\|^2 \alpha_{j^{*}} \cdot \mathbf{1}[j^{*} \in [s_\tau, e_\tau)] \cdot \mathbf{1}[x_{j^{*}} = y_t] = \|\mathbf{u}\|^2 \alpha_{j^{*}} \cdot \mathrm{Wu}^{(l,h)}_t$, which is positive iff $\mathrm{Wu}^{(l,h)}_t = 1$.
\end{proof}

\Cref{prop:wu} clarifies why Wu's score is a strong baseline on literal NIAH: under (W1)--(W3), it is a hard-thresholded version of \methodnameabbrev. It also clarifies the failure mode on NoLiMa: assumption (W1) collapses, since a non-literal retrieval head transforms the value vector through a learned $W_O \mathbf{v}_j$ that does not factor through $\mathbf{u}_{x_j}$. The Wu indicator drops to near zero (the attended token is not the answer token) while $c_j$ remains large, producing the $0.97 \to 0.03$ gap reported in \cref{sec:additional-details}.

\subsection{HeadKV/CompressKV as weighted attention accumulation}
\label[appendix]{app:rel:headkv}

\citet{fu2025headkv} (HeadKV) and \citet{lin2025compresskv} (CompressKV) score heads by accumulating attention mass on retrieval-relevant positions, optionally weighted. A schematic form, abstracting over implementation details, is
\begin{equation}\label{eq:headkv-score}
\mathrm{HeadKV}^{(l,h)} \;\propto\; \sum_{t \in \mathcal{A}^\tau} \sum_{j \in \mathcal{R}_t} w_j \cdot \alpha^{(l,h)}_{t,j},
\end{equation}
where $\mathcal{R}_t$ is the set of retrieval-relevant source positions at step $t$ and $w_j \geq 0$ are non-negative weights.

\begin{corollary}[HeadKV/CompressKV are non-negative reweightings of attention mass]
\label{cor:headkv}
\Cref{eq:headkv-score} is non-decreasing in each $\alpha^{(l,h)}_{t,j}$ for $j \in \mathcal{R}_t$ and depends on the OV circuit only through the choice of $\mathcal{R}_t$ and $w_j$. In particular, two heads with identical attention pattern $\alpha^{(l,h)}_{t,\cdot}$ but different output projections $W_O^{(l,h)}$ receive the same HeadKV/CompressKV score; \methodnameabbrev distinguishes them via $c_j$.
\end{corollary}
\begin{proof}
The first claim is immediate. The second follows because \cref{eq:headkv-score} is a function of $\alpha$ alone and the $\mathcal{R}_t, w_j$ choices are tied to the attention pattern, while \cref{eq:phi-cj} is a function of both $\alpha$ and $c_j$.
\end{proof}

\Cref{cor:headkv} formalises the central conceptual claim of the paper: any score that observes only the attention pattern, however weighted, is blind to differences in $W_O \mathbf{v}_j$ and therefore cannot distinguish the two heads sketched in \cref{fig:overview}(b).

\subsection{Spatial vs.\ temporal contrast}
\label[appendix]{app:rel:spatial-vs-temporal}

Attention-based methods commonly compute a \emph{temporal} contrast: attention mass accumulated at answer-token decode steps, optionally minus mass at non-answer steps~\citep{wu2025retrieval}. \methodnameabbrev uses a \emph{spatial} contrast (\cref{eq:needle-logit}). The two are not equivalent.

\begin{proposition}[Spatial contrast dominates under decode-step-stationary attention]
\label{prop:spatial-temporal}
Suppose head $(l, h)$ exhibits a constant attention pattern across decode steps, $\alpha^{(l,h)}_{t,j} = \alpha^{(l,h)}_{j}$ for all $t \in \mathcal{A}^\tau \cup \mathcal{A}^{\tau,\mathrm{neg}}$, where $\mathcal{A}^{\tau,\mathrm{neg}}$ is a set of non-answer steps and $|\mathcal{A}^\tau| = |\mathcal{A}^{\tau,\mathrm{neg}}|$. Then the temporal contrast at any single source position is zero, while the spatial contrast is generally non-zero whenever the head allocates more mass to the needle than the length-normalized off-needle expectation.
\end{proposition}
\begin{proof}
The temporal contrast at position $j$ is $\sum_{t \in \mathcal{A}^\tau} \alpha^{(l,h)}_{t,j} - \sum_{t \in \mathcal{A}^{\tau,\mathrm{neg}}} \alpha^{(l,h)}_{t,j} = (|\mathcal{A}^\tau| - |\mathcal{A}^{\tau,\mathrm{neg}}|) \cdot \alpha^{(l,h)}_j = 0$ by the cardinality assumption. The spatial contrast at any single step $t$ is $M^{+}_t - \tfrac{e_\tau - s_\tau}{N_t - (e_\tau - s_\tau)} M^{-}_t$, which is non-zero iff the per-position attention mass exceeds the off-needle average.
\end{proof}

\Cref{prop:spatial-temporal} captures the second motivation for the spatial contrast in \Cref{sec:method}: a head that persistently reads the needle---e.g., as ``context'' rather than ``answer''---receives no credit under temporal contrast but is correctly identified by spatial contrast when its OV write is answer-aligned at needle positions.

\subsection{Gradient interpretation}
\label[appendix]{app:rel:gradient}

We give the standard derivation of $\phi$ as a leading-order gradient. Let $\mathcal{L}_t = -\log p(y_t \mid \mathrm{context})$ denote the cross-entropy at step $t$, with $p(\cdot \mid \mathrm{context}) = \mathrm{softmax}(\boldsymbol{\ell}_t)$. Under the direct-path approximation $\boldsymbol{\ell}_t \approx W_U \mathbf{x}^{(L)}_t$ (\cref{app:direct-path:identity}), the chain rule gives
\begin{equation}\label{eq:gradient-phi}
\nabla_{\mathbf{o}^{(l,h)}_{t,j}} \mathcal{L}_t \;=\; W_U^\top \nabla_{\boldsymbol{\ell}_t} \mathcal{L}_t \;=\; W_U^\top \bigl( \mathbf{p}_t - \mathbf{e}_{y_t} \bigr) \;=\; -(1 - p(y_t)) \, \mathbf{u}_{y_t} \;+\; \!\sum_{v \neq y_t} \!p(v) \, \mathbf{u}_v,
\end{equation}
where $\mathbf{p}_t$ is the predicted distribution and $\mathbf{e}_{y_t}$ is the one-hot target. The dominant term is $-(1 - p(y_t)) \mathbf{u}_{y_t}$, a positive scalar multiple of $\mathbf{u}_{y_t}$ whenever $p(y_t) < 1$. Therefore the projection $\mathbf{u}_{y_t}^\top \mathbf{o}^{(l,h)}_{t,j} = \phi^{(l,h)}_{t,j}$ captures the leading-order direction of steepest descent of $\mathcal{L}_t$ in the $\mathbf{o}^{(l,h)}_{t,j}$ subspace.

\begin{corollary}[Gradient-faithful sign]
\label{cor:gradient-sign}
Under the direct-path approximation and $p(y_t) < 1$, increasing $\phi^{(l,h)}_{t,j}$ decreases $\mathcal{L}_t$ to first order; ablating a head with large $\phi^{(l,h)}_{t,j}$ should therefore raise the cross-entropy at $y_t$ by an amount proportional to $\phi^{(l,h)}_{t,j}$.
\end{corollary}

The corollary motivates the use of $\phi$ for head selection: heads with large positive $\phi$ are heads whose ablation should hurt the answer logit most, which is exactly the head set we want \methodnameabbrev to identify. Attention-based scoring, by contrast, projects the per-position output onto a uniform direction in $j$ (the all-ones combiner) and is gradient-free with respect to the answer token---which is why it generalizes poorly when the answer-aligned direction varies across attended positions.

\subsection{Summary}
\label[appendix]{app:rel:summary}

\Cref{tab:logit-vs-attention} consolidates the structural differences derived above.

\begin{table}[h]
\centering
\small
\caption{\textbf{Attention-based scoring is a special case of \methodnameabbrev}, recovered when the OV circuit contributes no position-dependent answer-aligned signal (\cref{prop:reduction}).}
\label{tab:logit-vs-attention}
\resizebox{\linewidth}{!}{%
\begin{tabular}{@{}lll@{}}
\toprule
\textbf{Property} & \textbf{Attention-based} & \textbf{Logit contribution (\methodnameabbrev)} \\
\midrule
Per-position observable        & $\alpha^{(l,h)}_{t,j}$ & $\alpha^{(l,h)}_{t,j} \cdot \mathbf{u}_{y_t}^\top W_O^{(l,h)} \mathbf{v}^{(l,h)}_{t,j}$ \\
Includes OV circuit            & no & yes \\
Distinguishes heads with equal $\alpha$ & no (Cor.~\ref{cor:headkv}) & yes \\
Contrast axis                  & temporal (answer vs.\ non-answer steps) & spatial (needle vs.\ off-needle positions) \\
Requires non-answer steps      & yes & no (Prop.~\ref{prop:spatial-temporal}) \\
Score sign                     & non-negative (clamped) & unclamped (negative $\Rightarrow$ off-needle-dominant; Cor.~\ref{cor:sign-reduction}) \\
Reduction                      & --- & equals attention mass when $W_O \mathbf{v}_j$ is position-independent \\
Wu/NIAH score                  & native & special case (Prop.~\ref{prop:wu}, conditions W1--W3) \\
HeadKV/CompressKV              & native & subsumed (Cor.~\ref{cor:headkv}) \\
Gradient interpretation        & projects onto $\mathbf{1}$ in $j$ & projects onto $\mathbf{u}_{y_t}$ in $\mathbb{R}^d$ (Eq.~\ref{eq:gradient-phi}) \\
\bottomrule
\end{tabular}
}
\end{table}

The reductions above predict three empirical patterns:
\begin{itemize}[nosep,leftmargin=*]
    \item On literal NIAH (assumptions W1--W3 hold), Wu/NIAH and \methodnameabbrev should rank heads similarly.
    \item On NoLiMa (assumption W1 fails), the rankings should diverge, with \methodnameabbrev assigning credit to heads whose $W_O \mathbf{v}_j$ is answer-aligned even when the attended token is not the answer token.
    \item Under causal validation, the divergent rankings should manifest as different ablation curves: heads identified only by \methodnameabbrev should be more causally critical on NoLiMa than heads identified only by Wu.
\end{itemize}
The Wu/NIAH-vs-NoLiMa gap reported in \cref{sec:additional-details} (top head: $0.97$ on NIAH, $0.03$ on NoLiMa, no causal effect under ablation) and the ablation-curve separation in \cref{fig:ablation-comparison} are consistent with all three predictions. The bottom-$k$ control in \cref{sec:bottom-k} additionally rules out the trivial alternative that any answer-aligned signal would suffice for causal effect: heads with large $|c_j|$ but off-needle-dominant attention mass leave NoLiMa ROUGE-L near baseline under ablation, exactly as \cref{cor:sign-reduction} predicts.

\section{Future Work}
\label{app:future-work}

Head-aware KV cache compression methods~\citep{zhang2023h2o,li2024snapkv,xiao2024streamingllm,fu2025headkv,lin2025compresskv} allocate per-head budgets from attention-based retrieval scores.
If the heads that matter for non-literal retrieval write rather than attend, then attention-based budgets will under-allocate cache to the heads doing the work and over-allocate it to heads whose attention is bookkeeping.
Substituting \methodnameabbrev as the scoring function is a direct test of this prediction at fixed cache budget, and \cref{app:kv-group-layer} sketches a per-(layer, KV-group) variant for grouped-query attention.
Context-faithful decoding methods~\citep{gema2025decore, ma2026interpretability} contrast a base model with one whose attention-identified retrieval heads are masked; masking \methodnameabbrev heads instead changes which circuit is suppressed and is a candidate for a stronger contrastive signal.


\newpage
\section*{NeurIPS Paper Checklist}

The checklist is designed to encourage best practices for responsible machine learning research, addressing issues of reproducibility, transparency, research ethics, and societal impact. Do not remove the checklist: {\bf The papers not including the checklist will be desk rejected.} The checklist should follow the references and follow the (optional) supplemental material.  The checklist does NOT count towards the page
limit. 

Please read the checklist guidelines carefully for information on how to answer these questions. For each question in the checklist:
\begin{itemize}
    \item You should answer \answerYes{}, \answerNo{}, or \answerNA{}.
    \item \answerNA{} means either that the question is Not Applicable for that particular paper or the relevant information is Not Available.
    \item Please provide a short (1--2 sentence) justification right after your answer (even for \answerNA). 
\end{itemize}

{\bf The checklist answers are an integral part of your paper submission.} They are visible to the reviewers, area chairs, senior area chairs, and ethics reviewers. You will also be asked to include it (after eventual revisions) with the final version of your paper, and its final version will be published with the paper.

The reviewers of your paper will be asked to use the checklist as one of the factors in their evaluation. While \answerYes{} is generally preferable to \answerNo{}, it is perfectly acceptable to answer \answerNo{} provided a proper justification is given (e.g., error bars are not reported because it would be too computationally expensive'' or ``we were unable to find the license for the dataset we used''). In general, answering \answerNo{} or \answerNA{} is not grounds for rejection. While the questions are phrased in a binary way, we acknowledge that the true answer is often more nuanced, so please just use your best judgment and write a justification to elaborate. All supporting evidence can appear either in the main paper or the supplemental material, provided in appendix. If you answer \answerYes{} to a question, in the justification please point to the section(s) where related material for the question can be found.

IMPORTANT, please:
\begin{itemize}
    \item {\bf Delete this instruction block, but keep the section heading ``NeurIPS Paper Checklist"},
    \item  {\bf Keep the checklist subsection headings, questions/answers and guidelines below.}
    \item {\bf Do not modify the questions and only use the provided macros for your answers}.
\end{itemize}


\begin{enumerate}

\item {\bf Claims}
    \item[] Question: Do the main claims made in the abstract and introduction accurately reflect the paper's contributions and scope?
    \item[] Answer: \answerYes{}
    \item[] Justification: The abstract and \cref{sec:introduction} state three scoped claims --- non-literal retrieval head detection via OV-circuit projection, causal validation across six model configurations, and retrieval specificity --- each backed by experiments in \cref{sec:ablation-comparison,sec:bottom-k,sec:specificity}.

\item {\bf Limitations}
    \item[] Question: Does the paper discuss the limitations of the work performed by the authors?
    \item[] Answer: \answerYes{}
    \item[] Justification: \cref{sec:limitations} discusses the off-needle-baseline sensitivity to distractor content, the restriction to attention heads (FFN sublayers not scored), and the coverage of architecture variations.

\item {\bf Theory assumptions and proofs}
    \item[] Question: For each theoretical result, does the paper provide the full set of assumptions and a complete (and correct) proof?
    \item[] Answer: \answerYes{}
    \item[] Justification: All propositions, lemmas, and corollaries are stated and proved in \cref{app:relationship-attention}, with assumptions made explicit in each statement (e.g., conditions W1--W3 in \cref{prop:wu}); supporting derivations for the direct-path approximation appear in \cref{app:direct-path}.

    \item {\bf Experimental result reproducibility}
    \item[] Question: Does the paper fully disclose all the information needed to reproduce the main experimental results of the paper to the extent that it affects the main claims and/or conclusions of the paper (regardless of whether the code and data are provided or not)?
    \item[] Answer: \answerYes{}
    \item[] Justification: \cref{sec:setup,sec:additional-details} specify the probing protocol, model checkpoints, decoding configuration, ablation procedure, and aggregation; \cref{sec:architecture-adaptations} details the architecture-specific implementation; \cref{app:assets} lists exact model versions and datasets.

\item {\bf Open access to data and code}
    \item[] Question: Does the paper provide open access to the data and code, with sufficient instructions to faithfully reproduce the main experimental results, as described in supplemental material?
    \item[] Answer: \answerYes{}
    \item[] Justification: \cref{app:code-data} describes the released repository, including detection scripts, the vLLM-based ablation driver, evaluation scripts, pre-computed head-score files, and a pinned environment specification.

\item {\bf Experimental setting/details}
    \item[] Question: Does the paper specify all the training and test details (e.g., data splits, hyperparameters, how they were chosen, type of optimizer) necessary to understand the results?
    \item[] Answer: \answerYes{}
    \item[] Justification: All hyperparameters (probing protocol, decoding, aggregation, ablation depths) are consolidated in \cref{tab:hyperparameters} of \cref{sec:additional-details}; trial filtering and held-out splits are specified in \cref{sec:setup}.

\item {\bf Experiment statistical significance}
    \item[] Question: Does the paper report error bars suitably and correctly defined or other appropriate information about the statistical significance of the experiments?
    \item[] Answer: \answerYes{}
    \item[] Justification: 95\% non-parametric bootstrap confidence intervals ($B{=}1{,}000$ resamples over passing trials) are reported for per-head scores $S_{l,h}$, as specified in \cref{sec:additional-details}; ablation curves are evaluated on a held-out set of 800 NoLiMa trials per condition.

\item {\bf Experiments compute resources}
    \item[] Question: For each experiment, does the paper provide sufficient information on the computer resources (type of compute workers, memory, time of execution) needed to reproduce the experiments?
    \item[] Answer: \answerYes{}
    \item[] Justification: \cref{app:compute} reports GPU type, tensor-parallel configuration, and approximate wall-clock for detection, calibration, and ablation, plus a total-project estimate that includes preliminary experiments not in the paper.
    
\item {\bf Code of ethics}
    \item[] Question: Does the research conducted in the paper conform, in every respect, with the NeurIPS Code of Ethics \url{https://neurips.cc/public/EthicsGuidelines}?
    \item[] Answer: \answerYes{}
    \item[] Justification: The work uses publicly available models and datasets within their licenses (\cref{app:assets}), involves no human subjects, and the broader-impact considerations are discussed in \cref{app:broader-impacts}.

\item {\bf Broader impacts}
    \item[] Question: Does the paper discuss both potential positive societal impacts and negative societal impacts of the work performed?
    \item[] Answer: \answerYes{}
    \item[] Justification: \cref{app:broader-impacts} discusses positive impacts (KV cache compression, hallucination mitigation, mechanistic interpretability), potential misuse (denial-of-capability attacks via head suppression) with mitigations, and fairness considerations regarding monolingual evaluation.
    
\item {\bf Safeguards}
    \item[] Question: Does the paper describe safeguards that have been put in place for responsible release of data or models that have a high risk for misuse (e.g., pre-trained language models, image generators, or scraped datasets)?
    \item[] Answer: \answerNA{}
    \item[] Justification: We release no new pre-trained models or scraped datasets; only diagnostic code and small author-constructed control sets (city--country, arithmetic) are released, none of which carry high misuse risk (\cref{app:broader-impacts,app:code-data}).

\item {\bf Licenses for existing assets}
    \item[] Question: Are the creators or original owners of assets (e.g., code, data, models), used in the paper, properly credited and are the license and terms of use explicitly mentioned and properly respected?
    \item[] Answer: \answerYes{}
    \item[] Justification: \cref{tab:assets} in \cref{app:assets} lists every model, dataset, and software library used, with version, identifier, and license; original papers are cited and terms of use are respected.

\item {\bf New assets}
    \item[] Question: Are new assets introduced in the paper well documented and is the documentation provided alongside the assets?
    \item[] Answer: \answerYes{}
    \item[] Justification: The released code, pre-computed head-score files, and author-constructed control sets are documented in \cref{app:code-data}, with a README and pinned environment included in the repository.

\item {\bf Crowdsourcing and research with human subjects}
    \item[] Question: For crowdsourcing experiments and research with human subjects, does the paper include the full text of instructions given to participants and screenshots, if applicable, as well as details about compensation (if any)?
    \item[] Answer: \answerNA{}
    \item[] Justification: The work involves no crowdsourcing or human subjects.

\item {\bf Institutional review board (IRB) approvals or equivalent for research with human subjects}
    \item[] Question: Does the paper describe potential risks incurred by study participants, whether such risks were disclosed to the subjects, and whether Institutional Review Board (IRB) approvals (or an equivalent approval/review based on the requirements of your country or institution) were obtained?
    \item[] Answer: \answerNA{}
    \item[] Justification: The work involves no human subjects, so no IRB review was required.

\item {\bf Declaration of LLM usage}
    \item[] Question: Does the paper describe the usage of LLMs if it is an important, original, or non-standard component of the core methods in this research? Note that if the LLM is used only for writing, editing, or formatting purposes and does \emph{not} impact the core methodology, scientific rigor, or originality of the research, declaration is not required.
    \item[] Answer: \answerNA{}
    \item[] Justification: LLMs are the object of study but are not part of the core methodology in any non-standard way; LLM use for writing assistance is declared in \cref{app:llm-usage}.

\end{enumerate}

\end{document}